\documentclass{article}
\usepackage{microtype}
\usepackage{graphicx, subfigure}
\usepackage{subcaption}
\usepackage{booktabs}
\usepackage{hyperref}

\usepackage[preprint]{icml2026}

\usepackage{amsmath, amssymb, mathtools, amsthm}

\usepackage[capitalize,noabbrev]{cleveref}

\theoremstyle{plain}
\newtheorem{theorem}{Theorem}[section]
\newtheorem{proposition}[theorem]{Proposition}
\newtheorem{lemma}[theorem]{Lemma}
\newtheorem{corollary}[theorem]{Corollary}
\newtheorem{example}[theorem]{Example}
\theoremstyle{definition}
\newtheorem{definition}[theorem]{Definition}

\theoremstyle{remark}
\newtheorem{remark}[theorem]{Remark}

\usepackage{accents}
\newcommand{\ubar}[1]{\underaccent{\bar}{#1}}

\newcommand{\expect}[1]{\mathbb{E}\left[ #1 \right]} 
 
\newcommand{\mexpect}[2]{\mathbb{E}_{#1}\left[ #2\right]}

\newcommand{\prob}[1]{\mathbb{P}\left( #1 \right)}

\newcommand{\supp}{\text{support}} 
\newcommand{\interior}{\text{int}}

\newcommand{\floor}[1]{\left\lfloor #1 \right\rfloor}
\newcommand{\abs}[1]{\left\vert #1 \right\vert}
\newcommand{\norm}[1]{\left\Vert #1 \right\Vert}
\newcommand{\oo}[1]{\left( #1 \right)}
\newcommand{\oc}[1]{\left( #1 \right]}

\newcommand{\cc}[1]{\left[ #1 \right]}

\newcommand{\NormRV}[2]{\ensuremath{\operatorname{N}\left(#1, #2\right)}}

\usepackage[textsize=tiny]{todonotes}

\icmltitlerunning{Unbiased Single-Queried Gradient for Combinatorial Objective}

\begin{document}

\twocolumn[
  \icmltitle{Unbiased Single-Queried Gradient for Combinatorial Objective}

  \begin{icmlauthorlist}
    \icmlauthor{Thanawat Sornwanee}{xxx}
  \end{icmlauthorlist}

  \icmlaffiliation{xxx}{Graduate School of Business, Stanford University, California, USA}

  \icmlcorrespondingauthor{Thanawat Sornwanee}{tsornwanee@stanford.edu}
  
  \icmlkeywords{Machine Learning, ICML, Optimization, Gradient, Orcale, Combinatorial}
  \vskip 0.3in
]

\printAffiliationsAndNotice{}

\begin{abstract}
    In a probabilistic reformulation of a combinatorial problem, we often face an optimization over a hypercube, which corresponds to the Bernoulli probability parameter for each binary variable in the primal problem. The combinatorial nature suggests that an exact gradient computation requires multiple queries. We propose a stochastic gradient that is unbiased and requires only a single query of the combinatorial function. This method encompasses a well-established REINFORCE (through an importance sampling), as well as including a class of new stochastic gradients.
\end{abstract}

\section{Introduction}

Many combinatorial optimization problems are easiest to interact with through a \textbf{pointwise oracle}: given a candidate binary vector $y \in \{0,1\}^d$, we can evaluate a quantity $Q(y)$, but we cannot efficiently access global structure of $Q$ (e.g., its neighborhood behavior or a tractable surrogate). This regime appears in integer programming, where $Q(y)$ may encode constraint violations or penalties: verifying whether a particular assignment violates constraints is straightforward, while understanding how to modify $y$ to reduce violations can be difficult in high dimensions. Another natural use case can be in preference optimization where have to query the user for feedback at each scenario.

A natural way to bring gradient-based optimization to this setting is to relax discrete decisions into probability space, which is continuous. We consider a product-Bernoulli relaxation by considering a random variable
\begin{align*}
    Y \sim P_x := \bigotimes \text{Bernoulli}\oo{x_i}
\end{align*}
with $x \in[0,1]^d$, and define the multilinear extension
\begin{align*}
    v(x) :&= \mexpect{Y \sim P_x}{Q(Y)}\\ &= \sum_{y\in \{0,1\}^d} \oo{Q(y) \prod_{i=1}^d \cc{x_iy_i + (1-x_i)(1-y_i)}}.
\end{align*}
Since the domain of $Q$ is finite, $v$ is well-defined, and infinitely smooth.\footnote{See the corollary~\ref{corollary:infdiff} and the corollary~\ref{corollary:infdiff} in the appendix.} Moreover, we will have that, for any $y \in \{0,1\}^d$, the distribution $P_y$ will just be a degenerate distribution with the value being $y$ itself, making $v(y) = Q(y)$.

We then essentially converting the combinatorial optimization of
\begin{align}
\label{problem:discrete}
    \min_{y \in \{0,1\}^d} Q(y)
\end{align}
into a continuous optimization over a smooth function
\begin{align}
\label{problem:continuous}
    \min_{x \in [0,1]^d} v(x).
\end{align}
We will have that, for any $x^*$ being an optimal solution of the continuous optimization problem~\ref{problem:continuous}, then any element in the support of the distribution $P_{x^*}$ will be an optimal $y^*$ to the combinatorial optimization problem~\ref{problem:discrete}.

In principle, gradient-based methods can optimize this relaxation, and then round back to a discrete solution. In practice, however, gradient estimation is the bottleneck when $Q$ is only available via queries. 

A vanilla gradient computation will require $2^d$ queries, which is the same as solving the original problem to begin with. A finite difference method will require $O(d)$ queries per step, which is prohibitive in large dimensions. This motivates our \textbf{``single-queried"} stochastic gradient: produce an unbiased estimate of the query-dependent gradient $g(x) := \nabla v(x)$ using only one oracle evaluation $Q(\cdot)$ per Monte Carlo sample. 

Concretely, we construct a single-queried stochastic evaluation whose pathwise derivative exists and whose expected gradient matches $\nabla v(x)$, enabling one-query gradient descent via autodiff. The algorithm we proposed is the \textbf{``Easy Stochastic Gradient (ESG)"} as shown in the algorithm~\ref{alg:ESG}. We will provide some conditions that will make this algorithm to return an unbiased estimator of the gradient for any instance of the combinatorial problem $Q(\cdot)$ in the section~\ref{section:sgrad}. Unlike preexisting methods like REINFORCE, the stochastic value before the gradient is taken in our ESG algorithm will also be an unbiased estimation of the value $v(x)$ itself.

\begin{algorithm}[tb]
\caption{Easy Stochastic Gradient (ESG) via AutoGrad}
\label{alg:ESG}
\begin{algorithmic}
    \STATE {\bfseries Input:} state $x$, oracle $Q(\cdot)$
    \FOR{$i \in \{1,2,\dots,d\}$}
        \STATE Embed $e_i \gets \hat{\sigma}^{-1}(x_i)$;
        \STATE Sample noise $\epsilon_i \overset{\text{iid}}{\sim} \sigma$;
        \STATE Compute stochastic score $z_i \gets e_i +\epsilon_i$;
        \STATE Compute stochastic query $k_i \gets \mathbf{1}_{z_i \ge 0}$;
        \STATE Compute $f_i \gets f\oo{\abs{z_i}}$
    \ENDFOR
    \STATE Query $q \gets Q(k)$;
    \STATE Compute stochastic value $v \gets q \times \prod_{i=1}^d f_i$;
    \STATE Compute stochastic gradient $g \gets \text{AutoGrad}(v, x)$.
    \STATE {\bfseries Return:}
    $(v,g)$
\end{algorithmic}
\end{algorithm}
\begin{figure}[ht]
  \begin{center}
    \centerline{\includegraphics[width=\columnwidth]{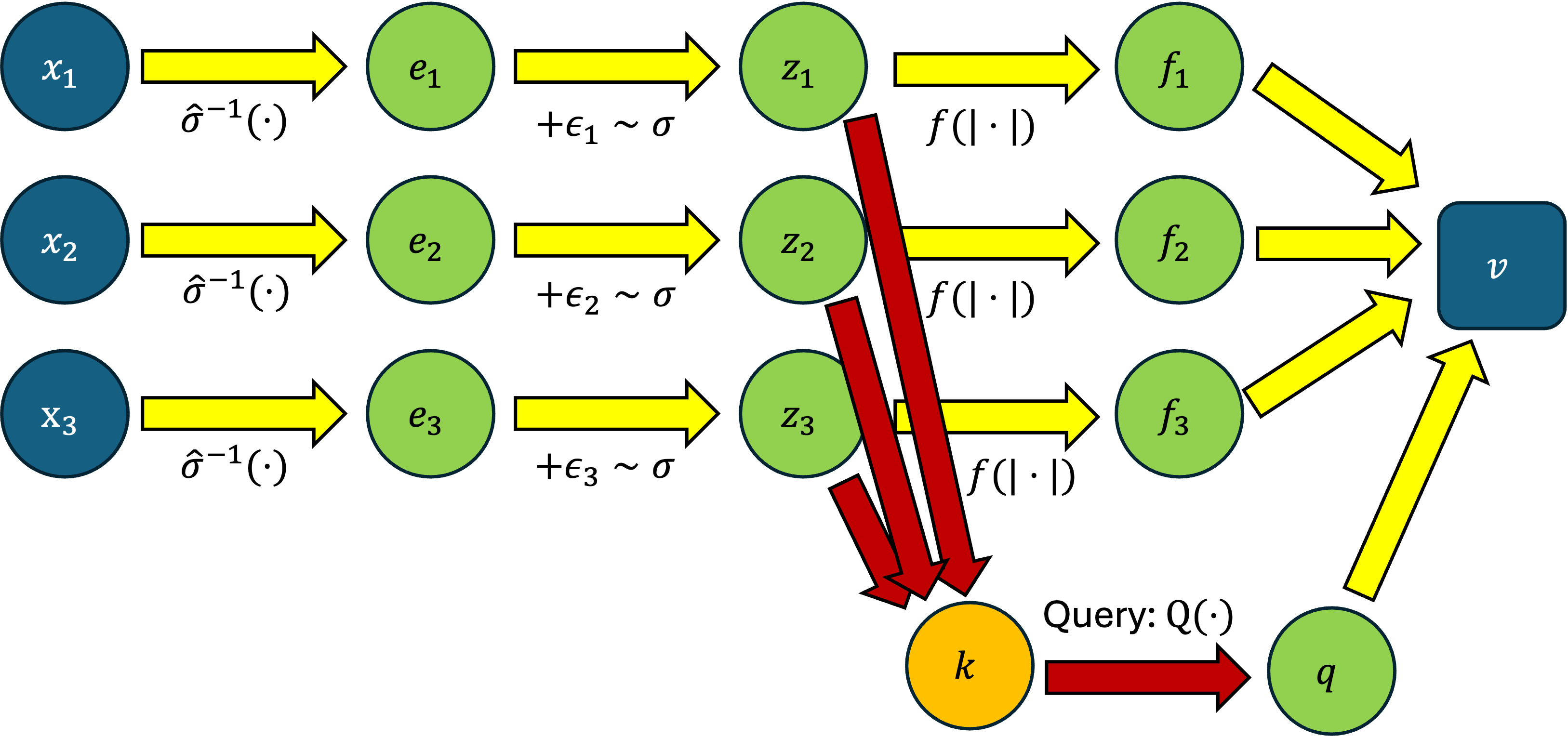}}
    \caption{
        Computational graph for our ESG algorithm~\ref{alg:ESG}: The yellow arrows represent operations that allow backpropagation of gradient, while the red one, which requires either rounding or querying, does not. The stochastic value $v$ itself can be taken gradient with respect to $x$, while yielding an unbiased gradient estimator.
    }
    \label{fig:ESG}
  \end{center}
\end{figure}

\paragraph{Why the obvious approach fails.}
It is easy to obtain an unbiased value estimator with one query: sample a key $K(x)\sim P_x$ (e.g., thresholding uniform noise) and return $V(x)=Q(K(x))$, so that $\expect{V(x)}=v(x)$. It is then tempting to think that, with sufficient regularity,
\begin{align*}
    \expect{\nabla V(x)} =\nabla \expect{V(x)} = \nabla v(x) = g(x),
\end{align*}
so, by choosing $G(x)$ to be $\nabla V(x)$, we have already solved the problem.
However, differentiating such an estimator pathwise is futile: $K(x)$ is discrete, so $\nabla Q(K(x))$ is either zero or undefined\footnote{In thresholding uniform noise method, it is zero almost surely since the boundary has zero measure.}, so
\begin{align*}
    \expect{\nabla V(x)} = 0 \ne \nabla \expect{V(x)}
\end{align*}
in general.

This is the core obstacle: we want single-query gradients, but the most natural single-query evaluations are not differentiable in a way that will be compatible to machine learning or optimization methods.

\section{Related Work}
\label{sec:related}

We study gradients of objectives of the form $\nabla_x \,\mathbb{E}_{\epsilon}[Q(K(x,\epsilon))]$ where $K$ produces \emph{discrete} random outcomes (e.g., rounding, argmax/top-$k$, combinatorial selections) and $Q$ is a black-box oracle evaluated \emph{only} on these discrete outcomes. This places our setting at the intersection of (i) gradient estimation for discrete random variables, (ii) continuous relaxations such as Gumbel-Softmax, and (iii) zeroth-order/black-box optimization.

\subsection{Score-function estimators and variance reduction}
A standard tool for discrete random variables is the score-function estimator, widely known as REINFORCE \cite{williams1992simple,schulman2015gradient}, which comes from the mathematical identity
\begin{align*}
    \nabla_x \mexpect{Y \sim P_x}{Q(Y)}
    =
    \mexpect{Y \sim P_x}{Q(Y) \nabla_x \log P_x(\{Y\})}.
\end{align*}
However, we will have that we cannot interpret $Q(Y)  P_x(\{Y\})$ as an unbiased estimator of $v(x)$, since we essentially perform differentiation on the probability space. Moreover, despite being unbiased, score-function estimators can suffer from high variance, motivating a large body of variance-reduction methods such as baselines and control variates \citep{schulman2015gradient}. Several influential estimators combine score-function gradients with learned or analytic control variates, including REBAR \citep{tucker2017rebar} and RELAX \citep{grathwohl2018relax}, which leverage continuous relaxations. For binary decisions, ARM \citep{yin2019arm} and DisARM \citep{dong2020disarm} provide unbiased estimators with strong variance reductions via augmentation and antithetic constructions.
Our work is close in spirit to this line—seeking unbiased (or provably controlled-bias) gradients—while targeting a different regime where the oracle $Q(\cdot)$ is only available at discrete solutions and oracle calls are the dominant cost.

\subsection{Continuous relaxations and straight-through gradients}
Another major approach is to replace discrete sampling/selection with a differentiable relaxation and backpropagate through the relaxed computation. The Gumbel-Softmax relaxation \citep{jang2017gumbel,maddison2017concrete} enables pathwise differentiation by sampling a continuous proxy that becomes discrete in a low-temperature limit. In practice, many implementations additionally apply a straight-through trick: use the discrete outcome in the forward pass, but backpropagate through the continuous relaxation in the backward pass \citep{bengio2013st, jang2017gumbel}. While often effective heuristically, straight-through trick produces a biased gradient estimator in general because the backward pass differentiates a surrogate computation that is not the one used in the forward pass \citep{bengio2013st}. This bias is especially problematic in our setting where $Q$ is only defined on the discrete output and cannot be meaningfully evaluated on the relaxed proxy, so the relaxation cannot directly serve as a faithful control variate without additional structure.

Recent papers~\cite{geng2025diffilo, li2025deft} have suggested that a relaxation to smooth value may work as long as $\nabla V(x)$ exists pathwise almost everywhere. However, the paper~\cite{sornwanee2026notdiff} shows that it is incorrect. Even in 1 dimensional case, the interchange for
\begin{align*}
    \frac{\partial}{\partial x} \expect{f(x, \epsilon)} = \expect{\frac{\partial}{\partial x} f(x, \epsilon)} 
\end{align*}
for all $x \in (0, 1)$ is not generally  possible. To get the interchangeability
of the differentiation, we normally require
\begin{align*}
    \expect{f(z, \epsilon) - f(y, \epsilon)}
    &= \expect{\int_{x=y}^z\frac{\partial}{\partial x}f(x, \epsilon) dx}\\
    &= \int_{x=y}^z \expect{\frac{\partial}{\partial x}f(x, \epsilon)} dx.
\end{align*}
This is generally not true when $f$ is discontinuous, which can happen when thresholding\footnote{which is similar to straight-through estimator where we have to select one discrete variable, often the hard maximum. This problem is then similar to the local optimization method in multi-level optimizations~\cite{jalota2024simple,sornwanee20251}.} is done. In our design, we will use a function such that, at the thresholding boundary, the value is always $0$ to avoid the discontinuity implication.

\subsection{Zeroth-order / black-box gradient methods}
When only function values are available, classical zeroth-order methods approximate gradients via finite differences, dating back to Kiefer-Wolfowitz stochastic approximation \citep{kiefer1952kw}. Simultaneous Perturbation Stochastic Approximation \citep{spall1992spsa} reduces query complexity by perturbing all coordinates simultaneously, and related random-direction estimators form the basis of modern derivative-free optimization. These methods typically require some smoothness assumptions\footnote{Even in a discrete grid case, the papers~\cite{wang2011discrete,11180056} require the approximation of the gradient by the midpoint gradient.} on the underlying objective and may still require multiple oracle calls per update for competitive variance.~\cite{bengio2013st} In contrast, our focus is on expectations induced by discrete random mappings $K(x)$ and on estimators tailored to this structure, with the goal of minimizing oracle calls.

\section{Stochastic Gradient}
\label{section:sgrad}

    The computation of the gradient $g(x)$ itself may require multiple query to the function $Q$, so we relax the requirement into finding a ``stochastic" gradient $G: (0,1)^d \to \mathbb{R}$ where $G(x)$ is defined upon the same random space for all $x \in (0,1)^d$ such that
    \begin{align*}
        \expect{G(x)} = g(x),
    \end{align*}
    while only require a single query of $Q$ per realization of $G(x)$. This also has to be universal across any oracle $Q: \{0,1\}^d \to \mathbb{R}$.

    Note that, we have also relaxed the requirement of the domain from $[0,1]^d$ to $(0,1)^d$, which will later be useful for defining the function. However, this relaxation is not problematic given that the corollary~\ref{corollary:infdiff} implies continuity of gradient.

        As we have seen earlier, an easy motivated example of the existence of such stochastic gradient is from that stochastic evaluation is trivial. However, if we want pathwise differentiation to exist in a correct way, we need more conditions, such as continuity after the thresholding.

        To give some sufficient condition, we will first rigorously define a stochastic valuation with single query, and then provide some condition to ensure that we can define a stochastic gradient with single query $G(x)$ to just be $\nabla V(x)$.

        Throughout, we consider $\oo{\Omega, \mathcal{F}, \mathbb{P}}$ to be a probability space, to be specified, and consider the function to be measurable with respect to $\oo{\Omega, \mathcal{F}}$ or jointly measurable with respect to an appropriate measurable space created by a combination of $\oo{\Omega, \mathcal{F}}$, Borel $\sigma$-algebra, and powerset of a finite set

        We want the function to be universal with respect to the query function $Q: \{0,1\}^d \to \mathbb{R}$, so we redefine $v: [0,1]^d \times \mathbb{R}^{\oo{\{0,1\}^d}} \to \mathbb{R}$ and $g: [0,1]^d \times \mathbb{R}^{\oo{\{0,1\}^d}} \to \mathbb{R}^d$ such that, for any $x \in[0,1]^d$, $Q: \{0,1\}^d \to \mathbb{R}$, 
        \begin{align*}
            v(x; Q) :&= \mexpect{Y \sim P_{x}}{Q(Y)}\\ &= \sum_{y\in \{0,1\}^d} \oo{Q(y) \prod_{i=1}^d \cc{x_iy_i + (1-x_i)(1-y_i)}},
        \end{align*}
        and
        \begin{align*}
            g(x; Q) := \nabla_x v(x; Q),
        \end{align*}
        which is well-defined given the corollary~\ref{corollary:infdiff}.

        \begin{definition}
        \label{def:sgrad}
            A function $G: (0,1)^d \times \Omega \times \mathbb{R}^{\oo{\{0,1\}^d}} \to \mathbb{R}^d$ is a single-queried stochastic gradient function with respect to $\oo{\Omega, \mathcal{F}, \mathbb{P}}$ if there exists some key function $K: (0,1)^d \times \Omega \to \{0,1\}^d$, and some main function $f_{\text{g}}: (0,1)^d \times \Omega \times \mathbb{R}$ such that, for every $x \in (0,1)^d$, and every $Q: \{0,1\}^d \to \mathbb{R}$,
            \begin{align*}
                G(x; \omega, Q) := f_{\text{g}}(x, \omega, Q(K(x; \omega)))
            \end{align*}
            for $\mathbb{P}$-almost surely $\omega \in \Omega$,
            \begin{align*}
                \expect{G(x; \omega, Q)} = g(x; Q).
            \end{align*}
        \end{definition}
        
        \begin{definition}
        \label{def:seval}
            A function $V: (0,1)^d \times \Omega \times \mathbb{R}^{\oo{\{0,1\}^d}} \to \mathbb{R}$ is a single-queried stochastic evaluation function with respect to $\oo{\Omega, \mathcal{F}, \mathbb{P}}$ if there exists some key function $K: (0,1)^d \times \Omega \to \{0,1\}^d$, and some main function $f_{\text{v}}: (0,1)^d \times \Omega \times \mathbb{R}$ such that, for every $x \in (0,1)^d$, and every $Q: \{0,1\}^d \to \mathbb{R}$,
            \begin{align*}
                V(x; \omega, Q) := f_{\text{v}}(x, \omega, Q(K(x; \omega)))
            \end{align*}
            for $\mathbb{P}$-almost surely $\omega \in \Omega$,
            \begin{align*}
                \expect{V(x; \omega, Q)} = v(x; Q).
            \end{align*}
        \end{definition}
        This definition suggests the following necessary and sufficient condition:
        \begin{proposition}
        \label{proposition:seval_characterization}
            With respect to $\oo{\Omega, \mathcal{F}, \mathbb{P}}$, $V$ is a stochastic evaluation function w with $f_v$ and $K$ as in the definition~\ref{def:seval} if and only if, for all $q \in \mathbb{R}$, $x \in (0,1)^d$, and $y \in \{0,1\}^d$,
            \begin{align*}
                \expect{f_{\text{v}}(x, \omega, q) \mathbf{1}_{K(x;\omega) = y}} = q \prod_{i=1}^d \cc{x_iy_i + (1-x_i)(1-y_i)}.
            \end{align*}
        \end{proposition}

        Recall that the gradient of the expected stochastic evaluation function may not be the same as the expected gradient of the stochastic evaluation function. If we want to have this property, we have to restrict it to a smaller class of stochastic functions.
        \begin{definition}
        \label{def:diffseval}
            A single-queried stochastic evaluation function $V$ is a gradientable stochastic evaluation function with respect to $\oo{\Omega, \mathcal{F}, \mathbb{P}}$ if, for any $x\in (0,1)^d$ and any $Q: \{0,1\}^d \to \mathbb{R}$, 
            \begin{align*}
                \expect{\nabla_x V(x; \omega, Q)} = \nabla_x v(x;Q),
            \end{align*}
            and $\nabla_x V(x; \omega, Q)$ exists for every $Q: \{0,1\}^d \to \mathbb{R}$, every $x \in (0,1)^d$, and $\mathbb{P}$-almost surely $\omega \in \Omega$, respectively.
        \end{definition}
        Note that, we require almost surely $\omega \in \Omega$ after fixing a value of $x \in (0,1)^d$ since, conditioned\footnote{in a do-algebra sense} on the arbitrary current value of $x \in (0,1)^d$, we want to be able to retrieve the stochastic gradient with probability $1$.

        \begin{corollary}
        \label{corollary:sgradfromdiffsval}
            With respect to $\oo{\Omega, \mathcal{F}, \mathbb{P}}$, if $V: (0,1)^d \times \Omega \times \mathbb{R}^{\oo{\{0,1\}^d}} \to \mathbb{R}$ is a single-queried gradientable stochastic evaluation function, then a function $G: (0,1)^d \times \Omega \times \mathbb{R}^{\oo{\{0,1\}^d}} \to \mathbb{R}^d$ defined on every $(x, \omega, Q)$ in the region where $\nabla_x V(x; \omega, Q)$ is well-defined to be
            \begin{align*}
                G(x; \omega, Q) = \nabla_x V(x; \omega, Q)
            \end{align*}
            is a single-queried stochastic gradient function.
        \end{corollary}
        \begin{proof}
            This follows directly from the definitions~\ref{def:sgrad} and~\ref{def:diffseval}.
        \end{proof}

        \begin{remark}[Stochastic Evaluation by Stochastic Key]
        \label{remark:stichastickey}
            We will see that we can create a simple construction of a stochastic evaluation by random evaluation such that
            \begin{align*}
                V(x) = Q(K(x))
            \end{align*}
            where the stochastic key function $K(x) \sim P_x$ for all $x \in (0,1)^d$. From the lemma~\ref{lemma:calibrated}, we can simply choose the stochastic key $K(x)$ to be $\cc{\mathbf{1}_{\sigma^{-1}(x_i)-\epsilon_i > 0}}_{i=1}^d$ for some invertible symmetric distribution $\sigma$ as defined in the appendix~\ref{appendix:1ddistribution}, and get that $\expect{V(x)} = v(x)$ for all $x \in (0,1)^d$ as wanted.
    
            However, as we have discussed, the gradient of $V(x)$ constructed by this method will not be the stochastic gradient $G(x)$.
        \end{remark}

        \begin{remark}[REINFROCE Gradient]
            Similar to the setting used in the remark~\ref{remark:stichastickey}, we can define a REINFORCE gradient
            \begin{align*}
                G(x) = Q\oo{K(x)} \log P_x\oo{\left\{K(x)\right\}}.
            \end{align*}
            This will be a single-queried stochastic gradient, but it is not a gradient of any single-queried gradientable stochastic evolution $V(x)$.
        \end{remark}

    \subsection{Easy Stochastic Gradient}

        We restrict the search to a simple formulation of the stochastic evaluation to find one that is gradientable. From the proposition~\ref{proposition:seval_characterization}, it is natural to consider the case when
        \begin{align*}
            f_{\text{v}}(x, \omega, q) = \hat{f}_{\text{v}}(x, \omega) q.
        \end{align*}
        We can see that this does not have to be the case, but it will help simplified the search, since now, we can just find $\hat{f}_{\text{v}}: (0,1)^d \times \Omega \to \mathbb{R}$ such that
        \begin{align*}
            \expect{\hat{f}_{\text{v}}(x, \omega) \mathbf{1}_{K(x;\omega) = y}} = \prod_{i=1}^d \cc{x_iy_i + (1-x_i)(1-y_i)}.
        \end{align*}
        The product structure and symmetry also suggests that we should consider $\hat{f}_{\text{v}}$ that is a product of a function of each coordinate. Thus, we consider $\Omega = (0,1)^d$ with $\Omega$ being a Borel $\sigma$-algebra with $\mathbb{P}$ being a Lebesgue measure, and consider the case when
        \begin{align*}
            \hat{f}_{\text{v}}(x, \omega) = \prod_{i=1}^d \hat{\hat{f}}_{\text{v}}(x_i, \omega_i),
        \end{align*}
        and
        \begin{align*}
            \mathbf{1}_{K(x;\omega) = y} = \prod_{i=1}^d  \mathbf{1}_{\hat{K}(x_i;\omega_i) = y_i}.
        \end{align*}
        We can then restrict the search to be finding $\hat{\hat{f}}_{\text{v}}$ and $\hat{K}$ such that
        \begin{align*}
            \expect{\hat{\hat{f}}_{\text{v}}(x_i, \omega_i) \mathbf{1}_{\hat{K}(x_i;\omega_i) = y_i}}= x_iy_i + (1-x_i)(1-y_i).
        \end{align*}

        Note that this search restriction is crucial since it converts a high-dimensional problem into a one-dimensional problem. This also suggests that, if we can find such $\oo{{\hat{\hat{f}}_{\text{v}}}, \hat{K}}$, we can reuse the same function for any different dimension problem. 
        
        Another condition that will aid both computation and search is to suggest that there exist some $f: \mathbb{R}_0^+ \to \mathbb{R}$ with $f(0) = 0$, $\hat{\sigma}, \sigma$ being a well-invertible symmetric distributions\footnote{We later relax that $\sigma$ needs to be well-invertible in the definition~\ref{def:goodtuple},} such that
        \begin{align*}
            \hat{\hat{f}}_{\text{v}}(x_i, \omega_i) = f\oo{\hat{\sigma}^{-1}\oo{x_i} +\sigma^{-1}\oo{\omega_i}},
        \end{align*}
        and
        \begin{align*}
            \hat{K}(x_i, \omega_i) = 
            \mathbf{1}_{\hat{\sigma}^{-1}\oo{x_i} +\sigma^{-1}\oo{\omega_i} \ge 0}.
        \end{align*}
        Thus, the requirement is
        \begin{align*}
            \expect{f\oo{\abs{\hat{\sigma}^{-1}\oo{x_i} +\sigma^{-1}\oo{\omega_i}}} \mathbf{1}_{\hat{\sigma}^{-1}\oo{x_i} +\sigma^{-1}\oo{\omega_i} \ge 0}}  = x_i.
        \end{align*} Furthermore, if $f$ is absolutely continuous with some minor regularity conditions, we will have that such $V$ is gradientable.

        We formalize it through the following definition and theorem.
        \begin{definition}
        \label{def:goodtuple}
            A tuple $(f, \hat{\sigma}, \sigma)$ is a good tuple if
            \begin{enumerate}
                \item $\sigma \in \Delta\oo{\mathbb{R}}$ is a symmetric distribution,
                \item $\hat{\sigma} \in \Delta\oo{\mathbb{R}}$ is a symmetric and positively-differentiable distribution,
                \item $f: \mathbb{R} \to \mathbb{R}$ is absolutely continuous with $f(x)=0$ for all $x\le 0$ and there exists a set $A \subseteq \mathbb{R}^+_0$ such that
                \begin{itemize}
                    \item $f(x)$ is bounded for all $x \in A$;
                    \item $f'(x)$ exists and is bounded for all $x \in A$;
                    \item $\sigma\oo{A + \{z\}} = 1$ for all $z \in \interior\oo{\supp\oo{\hat{\sigma}}}$,
                \end{itemize}
            \end{enumerate}
            and, for all $x \in (0,1)$,
            \begin{align*}
                \mexpect{\epsilon \sim \sigma}{f\oo{\hat{\sigma}^{-1}\oo{x} +\epsilon}} = x.
            \end{align*}
        \end{definition}

        In the case where $\sigma$ admits a pdf $\sigma'$, we will have that the condition is more or less equivalent to
        \begin{align*}
            \oo{f \ast \sigma'}(z) = \hat{\sigma}(z).
        \end{align*}
        
        Although we only evaluate $f$ in the non-positive region and set $f$ to be $0$ in non-positive region, we extend its definition to the negative region setting it to be $0$. This simplifies the notion of differentiation at $0$, ensuring that, if $\oo{\sigma \ast \delta_z}(\{0\})$ is non-zero, then the derivative has to be $0$. This will help with the indicator function, which is used to change the stochastic score to stochastic key as can be seen in the following theorem.

        \begin{theorem}
        \label{theorem:easy}
            With respect to $\oo{\Omega, \mathcal{F}, \mathbb{P}}$, for any good tuple $\oo{f, \sigma, \hat{\sigma}}$, if there exist random variables $\epsilon_i: \Omega \to \mathbb{R}$ for $i \in \{1,2,\dots, d\}$ such that $\epsilon_i \overset{\text{iid}}{\sim} \sigma$, then by defining stochastic score functions $Z_i: (0,1) \times \Omega \to \mathbb{R}$ for all $i \in \{1,2,\dots, d\}$ such that
            \begin{align*}
                Z_i(x_i; \omega) := \hat{\sigma}^{-1}(x_i) +\epsilon_i(\omega) 
            \end{align*}
            for all $x \in (0,1)^d$, $\omega \in \Omega$, a stochastic key function $K: (0,1)^d \times \Omega \to \{0,1\}^d$ such that
            \begin{align*}
                K(x; \omega) := \cc{\mathbf{1}_{Z_i(x_i, \omega) \ge 0}}_{i=1}^d
            \end{align*}
            for all $x \in (0,1)^d$, $\omega \in \Omega$, and a function $V: (0,1)^d \times \Omega \times \mathbb{R}^{\{0,1\}^d} \times \mathbb{R}$ such that
            \begin{align*}
                V(x; \omega, Q) := Q(K(x; \omega)) \prod_{i=1}^d f(\abs{Z_i(x_i; \omega)})
            \end{align*}
            for all $x \in (0,1)^d$, $\omega \in \Omega$, and $Q: \{0,1\}^d \to \mathbb{R}$, then $V$ is a single-queried gradientable stochastic evaluation function.

            Moreover, we have that $G: (0,1)^d \times \Omega \times \mathbb{R}^{\{0,1\}^d} \times \mathbb{R}^d$ defined such that, for all $x \in (0,1)^d$, $\omega \in \Omega$, and $Q: \{0,1\}^d \to \mathbb{R}$,
            \begin{align*}
                G(x; \omega, Q) := 
                Q(K(x; \omega)) 
                \nabla_x \cc{\prod_{i=1}^d f(\abs{Z_j(x_j; \omega)})}
            \end{align*}
            will be a single-queried stochastic gradient function.\footnote{Note that to compute $\nabla_x \cc{\prod_{j=1}^d f(\abs{Z_j(x_j; \omega)})}$, we can do it using autodifferentiation, or a mathematical fact that $\frac{\partial}{\partial x_i}f(\abs{Z_i(x_i; \omega)}) = \frac{
                        \mathbf{1}_{Z_i(x_i; \omega)> 0}f'\oo{\abs{Z_i(x_i; \omega)}}-\mathbf{1}_{Z_i(x_i; \omega)< 0}f'\oo{\abs{Z_i(x_i; \omega)}}
                    }
                    {
                        \hat{\sigma}'\oo{\hat{\sigma}^{-1}(x_i)}
                    }$.} 

            We denote such single-queried stochastic evaluation function $V$ and gradient function $G$ as an easy stochastic evaluation function and an easy stochastic gradient function, respectively.
        \end{theorem}
        Since we only use $\sigma$ for sampling, this theorem also tells us why we impose less condition over $\sigma$ comparing to $\hat{\sigma}$ in the definition~\ref{def:goodtuple}. We will also see that the continuity imposed in the definition~\ref{def:goodtuple} is crucial. Otherwise, we can use $f(z) := \mathbf{1}_{z \ge 0}$ and $\sigma = \hat{\sigma}$, and revert back to the case when $V(x; \omega) = Q(K(x; \omega))$ with $K(x) \sim P_x$ by the lemma~\ref{lemma:calibrated}.

    \subsection{Examples}
    \label{subsection:examples}

        Below we provide some examples of good tuples.

        For example, it is easy to construct one when $\sigma = \text{Unif}\oo{\cc{-\frac{1}{2}, \frac{1}{2}}}$.

        \begin{example}
            [Spike]
            Consider
            \begin{align*}
                f(z) := \begin{cases}
                    0 &\text{ if } z\le 0\\
                    4z &\text{ if } 0 \le z \le \frac{1}{2}\\
                    4(1-z) &\text{ if } \frac{1}{2} \le z \le 1\\
                    0 &\text{ if } z\ge 1
                \end{cases}
            \end{align*}
            for all $z \in \mathbb{R}$, and $\sigma := \text{Unif}\oo{\cc{-\frac{1}{2}, \frac{1}{2}}}$. Thus, for all $z \in \cc{-\frac{1}{2}, \frac{1}{2}}$,
            \begin{align*}
                \hat{\sigma}(z)
                =
                \begin{cases}
                    2\oo{\frac{1}{2}-\abs{x}}^2 &\text{ if } z \in \cc{-\frac{1}{2}, 0}\\
                    1-2\oo{\frac{1}{2}-\abs{x}}^2 &\text{ if } z \in \cc{0,\frac{1}{2}}
                \end{cases}.
            \end{align*}
        \end{example}

        \begin{example}
            [Arch]
            Consider
            \begin{align*}
                f(z) := \begin{cases}
                    0 &\text{ if } z\le 0\\
                    \frac{\pi}{2}\sin(\pi z) &\text{ if } z \ge 0
                \end{cases}
            \end{align*}
            for all $z \in \mathbb{R}$, and $\sigma := \text{Unif}\oo{\cc{-\frac{1}{2}, \frac{1}{2}}}$. Thus, for all $z \in \cc{-\frac{1}{2}, \frac{1}{2}}$,
            \begin{align*}
                \hat{\sigma}(z)
                =
                \frac{1+\sin(\pi z)}{2}.
            \end{align*}
        \end{example}
        Note that $f$ is not differentiable only at $0$. However, we can also find $f$ that is differentiable everywhere.

        \begin{example}
            [Cosine]
            Consider
            \begin{align*}
                f(z) := \begin{cases}
                    0 &\text{ if } z\le 0\\
                    1-\cos(2\pi z) &\text{ if } z \ge 0
                \end{cases}
            \end{align*}
            for all $z \in \mathbb{R}$, and $\sigma := \text{Unif}\oo{\cc{-\frac{1}{2}, \frac{1}{2}}}$. Thus, for all $z \in \cc{-\frac{1}{2}, \frac{1}{2}}$,
            \begin{align*}
                \hat{\sigma}(z)
                =
                z + \frac{1}{2} + \frac{\sin(2\pi z)}{2\pi}.
            \end{align*}
        \end{example}

        We can see that this is not specific to bounded distribution.
        \begin{example}
            [Bi-Gaussian with Cosine]
            Consider
            \begin{align*}
                f(z) := \begin{cases}
                    0 &\text{ if } z\le 0\\
                    1-\cos\oo{\frac{z}{2}} &\text{ if } z \ge 0
                \end{cases}
            \end{align*}
            for all $z \in \mathbb{R}$, and $\sigma := \frac{1}{2}\NormRV{\pi}{1} + \frac{1}{2}\NormRV{-\pi}{1}$. It is hard to find a closed form for $\hat{\sigma}$, but we can show that $\hat{\sigma}$ is an increasing everywhere CDF is symmetric.
        \end{example}

        So far, we have only consider a bounded positively-differentiable $\sigma$. However, we can also easy example, which is the easiest among the examples, where $\sigma = \text{Unif}\oo{\{-1,1\}}$.
        \begin{example}
            [Long Jump]
            Consider
            \begin{align*}
                f(z) := 2z-1
            \end{align*}
            for all $z \in \cc{\frac{1}{2}, \frac{3}{2}}$, and $\sigma := \text{Unif}\oo{\{-1,1\}}$. Thus, we will have that $\hat{\sigma} = \text{Unif}\oo{\cc{-\frac{1}{2}, \frac{1}{2}}}$.
        \end{example}
    
\section{Calibration}
\label{section:desiderata}

    \subsection{Easy Stochastic Gradient is not Calibrated}
    
        A single-queried stochastic evaluation function $V$ and a single-queried stochastic gradient function $G$ requires a single query through the stochastic key $K$. 
    
        A possible conjecture as well as desiderata in some scenario is that $K(x) \sim P_x$, meaning that we check the query in the same way as the nature would have done.
    
        \begin{definition}
            (Calibrated Key)
            With respect to $\oo{\Omega, \mathcal{F}, \mathbb{P}}$,
            a single-queried stochastic evaluation function $V$ is calibrated if there is a key function $K$ via the definition~\ref{def:seval}, and
            \begin{align*}
                K(x) \sim P_x
            \end{align*}
            for all $x \in (0,1)^d$.
            Similarly, a single-queried stochastic gradient function $G$ is calibrated if there is a key function $K$ via the definition~\ref{def:sgrad}, and $K(x) \sim P_x$ for all $x \in (0,1)^d$.
        \end{definition}

        Thus, the stochastic evaluation of a form $V(x; \omega) = Q(K(x;\omega))$ is always calibrated. However, so far, we have not established whether there exists a calibrated single-queried gradientable stochastic evaluation function. To explore this, we start by consider a class of easy stochastic evaluation function as introduced in the theorem~\ref{theorem:easy}. We have that whether such evaluation is calibrated can be simply tested from the underlying good tuple $\oo{f, \sigma, \hat{\sigma}}$.
    
        \begin{corollary}
        \label{corollary:sigsig}
            A good stochastic evaluation function $V$ with a good tuple $\oo{f, \sigma, \hat{\sigma}}$ is calibrated if and only if $\sigma = \hat{\sigma}$.
        \end{corollary}
        \begin{proof}
            This comes directly from the lemma~\ref{lemma:calibrated}, and how the key function is determined in the theorem~\ref{theorem:easy}.
        \end{proof}
        Note that every example of easy stochastic evaluation function we have given in the subsection~\ref{subsection:examples} does not satisfies $\sigma = \hat{\sigma}$. Moreover, we can even show that a good tuple with $\sigma = \hat{\sigma}$ does not exist as long as we enforce compact support or decreasing hazard rate.\footnote{We do conjecture that the result also holds for general symmetric distribution, including heavy tail case. However, in practice, we rarely use heavy tail for the encoding.}
        \begin{lemma}
        \label{lemma:no_fsigsig}
            There is no good tuple $\oo{f, \sigma, \hat{\sigma}}$ with $\sigma = \hat{\sigma}$ with compact support or with decreasing hazard rate.
        \end{lemma}
        The main intuition is that, because $\oo{f \ast \sigma'}(z) = \sigma(z)$, then with enough regularity,
        \begin{align*}
            \oo{f' \ast \sigma'}(z) = \sigma'(z),
        \end{align*}
        suggesting that $f'$ should be Dirac delta like, contradicting with that $f$ is absolutely continuous.
    
        \begin{proposition}
            There is no calibrated easy stochastic evaluation function, and no calibrated easy stochastic gradient function if the encoding function comes from compactly supported distribution or a distribution with decreasing hazard rate.
        \end{proposition}
        \begin{proof}
            The proof comes from the corollary~\ref{corollary:sigsig} and the lemma~\ref{lemma:no_fsigsig}.
        \end{proof}

        This shows that we cannot get unbiased, which is inherently in the definition, calibrated, and easy stochastic evaluation function, at least upto regularity on the distribution. 

        However, we know that there exists an unbiased and calibrated single-queried gradientable stochastic evaluation function. For example, REINFORCE algorithm gives a calibrated single-queried stochastic gradient, but it cannot be written as a gradient of a gradientable evaluation.

    \subsection{Variance, Calibration, \& Importance Sampling}

        Although the vanilla ESG is not calibrated, we can use importance sampling in order to calibrate it. For example, REINFORCE gradient can be viewed as an importance sampling of the Easy Stochastic Gradient (ESG)~\ref{alg:ESG} with good tuples $\oo{f, \hat{\sigma}, \sigma}$ called ``Long Jump" in the subsection~\ref{subsection:examples}, and vice versa. However, it is unclear whether we should always use the calibrated version.

        For example, in a one-dimensional case, if $Q(0)=0$ and $Q(1)=1$, we will have that the ground truth gradient $g(x)=1$. Both REINFORCE gradient and ESG: LongJump are unbiased. However, the variance of REINFORCE gradient is $\frac{1-x}{x}$, which explodes to $\infty$ when $x \to 0$, while the gradient obtained via the ESG algorithm is $1$.\footnote{The use of pathwise gradient may be responsible for the variance reduction in some regime.~\cite{asmussen2007stochastic}}

        Another use of importance sampling is also to reduce the variance of the step size. However, importance sampling normally employs some change of measure which is $x$ dependent, so we need to add additional term to the stochastic gradient computation. An example can be seen in the appendix~\ref{appendix:ESGREINFORCE}.

\section{Stochastic Query Descent}

    Now, we introduce and test stochastic query descent (SQD) in the algorithm~\ref{alg:SQD}, which is similar to a vanilla gradient descent but with Easy Stochastic Gradient (ESG) or other unbiased single-queried gradient used in place of the actual gradient.

    \begin{algorithm}[tb]
    \caption{Single Query Descent (SQD)}
    \label{alg:SQD}
    \begin{algorithmic}
        \STATE {\bfseries Input:} initial state $x^{(0)}$, learning rate schedule $\eta$, time $T$, oracle $Q(\cdot)$, and EasyStochasticGradient $\text{ESG}\oo{\cdot:\cdot}$
        \FOR{$t \in {1,2,\dots, T}$}
            \STATE $g^{(t)} \gets \text{ESG}\oo{x^{(t)};Q(\cdot)}$;
            \STATE Update $x^{(t+1)} \gets x^{(t)} - \eta_t g^{(t)}$.
        \ENDFOR
        \STATE {\bfseries Return:}
        $x_T$.
    \end{algorithmic}
    \end{algorithm}

    We provide some experiments as a proof-of-concept.

    \subsection{Experiment: Symmetric-Slice Optimization}
    \label{subsection:symmetricslice}

        In this experiment, we consider the following combinatorial maximization\footnote{Thus, the algorithm needs to be an ascent algorithm instead of descent algorithm.} problem, where the payoff $Q(\cdot)$ only depends on the hamming weight
        \begin{align*}
            S(y) = \sum_{i=1}^d y_i
        \end{align*}
        of the assignment $y \in \{0,1\}^d$. We define the oracle
        \begin{align*}
            Q(y)
            =
            \begin{cases}
                3&\text{ if }S(y)=d\\
                18&\text{ if }\abs{S(y)-\floor{\frac{d}{2}}}\le \floor{0.133d}\\
                -2&\text{ if }S(y)\le \floor{0.233d}\\
                0&\text{ otherwise}
            \end{cases}
        \end{align*}
        for any $y \in \{0,1\}^d$. 
        
        This construction creates a broad, attractive plateau near $S\approx d/2$. For all configurations $y$ whose weights fall in the region, the oracle returns exactly the same value. Consequently, many samples produce identical feedback, which can induce wide attractive regions and make local-progress behavior easy to interpret. The construction also includes a rare-event spike at $y=\mathbf{1}$, which is when $S(y) =d$.

        We perform the experiments with $d \in \{10,20,30\}$.

        \paragraph{Preliminary Result}

        We compare the result of our newly introduced algorithm Single Query Descent (SQD) with $3$ variants, ``SQD: Arch", ``SQD: Spike", and ``SQD: LongJump", against the 4 baselines in the literature: ``REINFORCE", ``RELAX", ``ARM", and ``DisARM". The first three (SQDs) implements Easy Stochastic Gradient (ESG)~\ref{alg:ESG} with good tuples $\oo{f, \hat{\sigma}, \sigma}$ from the 3 examples with corresponding names in the subsection~\ref{subsection:examples}. These three algorithms are the one introduced in our paper.

        The trajectories\footnote{We measure time as a number of oracle calls. ``ARM" and ``DisARM", which will require 2 oracle queries per iteration. Others use 1 query per iteration.} of the algorithms tend to be noisy, so we keep track of the best-so-far instead. From the result\footnote{Results for $d=10$ and $d=20$ are shown in the appendix~\ref{appendix:result}.} shown in the figure~\ref{fig:slice30}, we can see that REINFORCE tends to be stuck at the local minima, yielding the value of only $3$. ARM and DisARM tends to have the same problem. Our SQD algorithms take some iterations before rapidly increase the objective to the global maximum. RELAX tends to perform better, but this requires an additional learning for the control variate.  

        \begin{figure}[ht]
          \begin{center}
            \centerline{\includegraphics[width=\columnwidth]{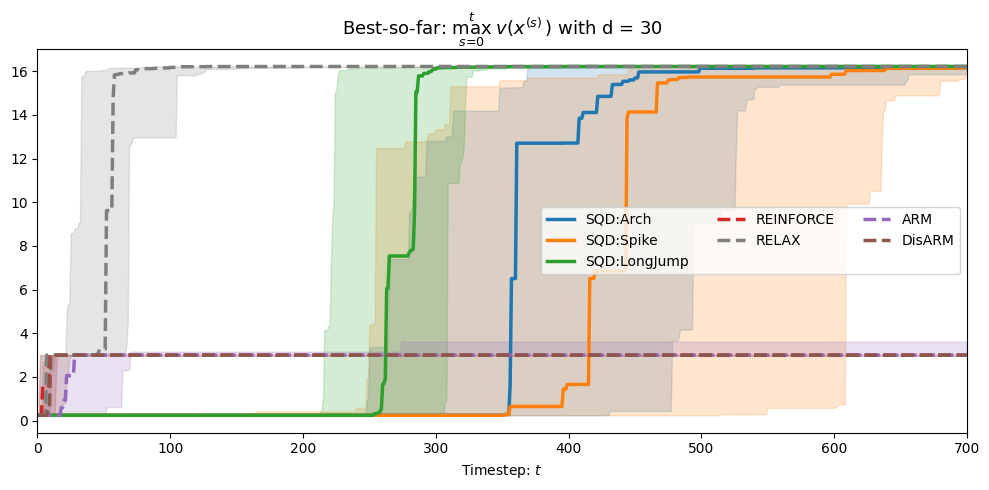}}
            \caption{
                Median of the best-so-far progress of different algorithms with inter-quartiles uncertainty. The algorithm we proposed are in bold lines, and the pre-existing algorithms are in dashed lines. This result comes from the case when the dimension $d= 30$.
            }
            \label{fig:slice30}
          \end{center}
        \end{figure}

    \subsection{Encoded ESG \& Encoded SQD}
    \label{subsection:ecnodedd}

        From our ESG algorithm~\ref{alg:ESG}, we can see that the computation will only use the information of the state $x \in (0,1)^d$ through the embedding $e \in \text{int}\oo{\text{support}\oo{\hat{\sigma}}}^d$. This suggests that, instead of converting the state $x$ to the embedding $e$ by applying $\oo{\hat{\sigma}}^{-1}$ in an elementwise manner, we can just perform the SQD dynamic directly on $e$ space and convert it back to $x$ after the iterations are done.

        These encoded versions\footnote{See the appendix~\ref{appendix:ecnodedd} for the encoded ESG algorithm~\ref{alg:EncodedESG} and the ecnoded SQD algorithm~\ref{alg:EncodedSQD}.} may reduce the autodifferentiation work load since the variable $z$ is linear in $e$. Moreover, the variance may be reduced when $x$ is close to the boundary since we do not require the division by $\hat{\sigma}'\oo{\hat{\sigma}^{-1}(x_i)}$. From the experiment\footnote{as shown in the appendix~\ref{appendix:result}}, we observe that the trajectory tends to be more gradual.
    
\section{Conclusion}

We proposed \textbf{Easy Stochastic Gradient (ESG)}, a new class of unbiased, single-queried gradient estimators for optimizing combinatorial objectives through their product-Bernoulli (multilinear) relaxation. The key idea is to construct a value estimator that is also pathwise differentiable, which reduce variance in some regime. Our framework unifies and connects to classical score-function methods (REINFORCE) via importance sampling, and it naturally generates a family of new estimators through the choice of underlying ``good tuples". In proof-of-concept experiments, \textbf{Single Query Descent (SQD)} instantiated with ESG achieved strong empirical performance relative to established unbiased estimators, while maintaining the critical property that each iteration requires only a single oracle evaluation.

Several directions remain open. First, developing \textbf{adaptive selection rules} that choose among unbiased estimators (or tune tuple families) based on the current state and observed history may further reduce variance and improve robustness across tasks. Second, extending the construction beyond product-Bernoulli structure could broaden applicability. Together, these results suggest that one-query, unbiased, autodiff-compatible optimization is a practical and promising avenue for large-scale combinatorial problems where oracle calls dominate cost.

\section*{Acknowledgements}

    The author thanks Rajat Dwaraknath, Haozhan Gao, Zijie Geng, Xuchen Gong, Lutong Hao, Senem Isik, Jiwon Jang, Chartsiri Jirachotkulthorn, Minseo Kim, Anish Senapati, James Spall, Ellen Vitercik, Jenny Xu, and Yunzhi Zhang for useful discussions.

\section*{Impact Statement}

    This paper presents work whose goal is to advance the field of Machine
    Learning. There are many potential societal consequences of our work, none
    which we feel must be specifically highlighted here.

\bibliography{bibliography}
\bibliographystyle{icml2026}

\newpage

\appendix
\onecolumn
\part*{Appendix}

\section{Encoded ESG \& Encoded SQD Algorithms}
\label{appendix:ecnodedd}
    Below, we provide algorithms for encoded ESG (algorithm~\ref{alg:EncodedESG}) and encoded SQD (algorithm~\ref{alg:EncodedSQD}) as mentioned in the subsection~\ref{subsection:ecnodedd}.

\begin{algorithm}[tb]
        \caption{Encoded Easy Stochastic Gradient (EncodedESG) via AutoGrad}
        \label{alg:EncodedESG}
        \begin{algorithmic}
            \STATE {\bfseries Input:} encoded state $e$, oracle $Q(\cdot)$
            \FOR{$i \in \{1,2,\dots,d\}$}
                \STATE Sample noise $\epsilon_i \overset{\text{iid}}{\sim} \sigma$;
                \STATE Compute stochastic score $z_i \gets e_i +\epsilon_i$;
                \STATE Compute stochastic query $k_i \gets \mathbf{1}_{z_i \ge 0}$;
                \STATE Compute $f_i \gets f\oo{\abs{z_i}}$
            \ENDFOR
            \STATE Query $q \gets Q(k)$;
            \STATE Compute stochastic value $v \gets q \times \prod_{i=1}^d f_i$;
            \STATE Compute stochastic gradient $g \gets \text{AutoGrad}(v, e)$.
            \STATE {\bfseries Return:}
            $(v,g)$
        \end{algorithmic}
        \end{algorithm}

\begin{algorithm}[tb]
        \caption{Encoded Single Query Descent (EncodedSQD)}
        \label{alg:EncodedSQD}
        \begin{algorithmic}
            \STATE {\bfseries Input:} initial state $x^{(0)}$, learning rate schedule $\eta$, time $T$, oracle $Q(\cdot)$, and EncodedEasyStochasticGradient $\text{EncodedESG}\oo{\cdot:\cdot}$
            \FOR{$i \in \{1,2,\dots,d\}$}
                \STATE Encode $e_i^{(0)} \gets \hat{\sigma}^{-1}\oo{x_i^{(0)}}$;
            \ENDFOR
            \FOR{$t \in {1,2,\dots, T}$}
                \STATE $g^{(t)} \gets \text{EncodedESG}\oo{e^{(t)};Q(\cdot)}$;
                \STATE Update $e^{(t+1)} \gets e^{(t)} - \eta_t g^{(t)}$.
            \ENDFOR
            \FOR{$i \in \{1,2,\dots,d\}$}
                \STATE Decode $x_i^{(T)} \gets \hat{\sigma}\oo{e_i^{(T)}}$;
            \ENDFOR
            \STATE {\bfseries Return:}
            $x^{(T)}$.
        \end{algorithmic}
        \end{algorithm}

\section{Bernoulli Reparametrization}

    These first two corollaries are straight-forward and explain the basic property of the function $v$.

    \begin{corollary}
    \label{corollary:lipchitz}
        The function $v: [0,1]^d \to \mathbb{R}$ is well-defined and Lipschitz continuous.
    \end{corollary}
    \begin{proof}
        First, we note that $\sup_{y \in \{0,1\}^d} \abs{Q(y)} < \infty$, so, for any $x \in [0,1]^d$,
        \begin{align*}
            \mexpect{Y \sim P_x}{\abs{Q(Y)}}
            \le
            \sup_{y \in \{0,1\}^d} \abs{Q(y)} < \infty,
        \end{align*}
        making $V$ well-defined and bounded.

        For any $x, x' \in [0,1]^d$, we have that, for all $i \in  \{1,2,\dots, d\}$,
        \begin{align*}
            \norm{\cc{P_{x'}}_i - \cc{P_x}_i}_{\text{TV}} = 
            \norm{\text{Bernoulli}(x'_i) - \text{Bernoulli}(x_i)}_{\text{TV}} = \abs{x'_i - x_i},
        \end{align*}
        where $\cc{\mu}_i$ denotes the marginal distribution of $\mu \in \Delta\oo{\cc{0,1}^d}$ on the $i^{\text{th}}$ element. From the independence, we will have that
        \begin{align*}
            P_x = \bigotimes_{i=1}^d \cc{P_x}_i,
        \end{align*}
        making
        \begin{align*}
            \norm{P_{x'} - P_x}_{\text{TV}} 
            =
            \norm{\bigotimes_{i=1}^d \cc{P_{x'}}_i - \bigotimes_{i=1}^d \cc{P_x}_i}_{\text{TV}}
            \le 
            \sum_{i=1}^d \norm{\cc{P_{x'}}_i - \cc{P_x}_i}_{\text{TV}} = \norm{x'-x}_1.
        \end{align*}
        Thus, we have that
        \begin{align*}
            \abs{v(x')-v(x)} \le \cc{\sup_{y \in \{0,1\}^d} Q(y) - \inf_{y \in \{0,1\}^d} Q(y)} \norm{P_{x'} - P_x}_{\text{TV}} \le \cc{\sup_{y \in \{0,1\}^d} Q(y) - \inf_{y \in \{0,1\}^d} Q(y)} \norm{x'-x}_1,
        \end{align*}
        so $v$ is Lipschitz continuous.
    \end{proof}

    \begin{corollary}
    \label{corollary:infdiff}
        The function $v$ is infinitely and uniformly bounded differentiable, meaning that there exists $K \in \mathbb{R}$ such that $\abs{\frac{\partial^{\abs{\alpha}}}{\partial x^\alpha} v(x)} < K$ for all $\alpha \in \oo{\mathbb{N}_0}^d$.
    \end{corollary}
    \begin{proof}
        We have that
        \begin{align*}
            \frac{\partial}{\partial x_1} v(x) =
            \frac{\partial}{\partial x_1}
            \cc{
                x_1 v\oo{\oo{1, \cc{x_i}_{i=2}^{d}}}+
                \oo{1-x_1} v\oo{\oo{0, \cc{x_i}_{i=2}^{d}}}
            }
            =
            v\oo{\oo{1, \cc{x_i}_{i=2}^{d}}}
            -
            v\oo{\oo{0, \cc{x_i}_{i=2}^{d}}},
        \end{align*}
        so
        \begin{align*}
             \frac{\partial}{\partial x_1^2} v(x) = 0.
        \end{align*}
        Similarly, we will have that, if $\norm{\alpha}_{\infty} \ge 2$, then $\frac{\partial^{\abs{\alpha}}}{\partial x^\alpha} v(x) = 0$.

        Using induction, we can also show that, if $\alpha = \cc{\mathbf{1}_{i \le k}}_{i=1}^d$ with $k\le d$, then 
        \begin{align*}
            \frac{\partial^{\abs{\alpha}}}{\partial x^\alpha} v(x)
            =
            \sum_{y \in \{0,1\}^k}
            \cc{
            \oo{
            \prod_{i=1}^k(2y_i-1)
            }
            v\oo{\oo{y, \cc{x_i}_{i=k+1}^{d}}}},
        \end{align*}
        so
        \begin{align*}
            \abs{\frac{\partial^{\abs{\alpha}}}{\partial x^\alpha} v(x)}
            \le
            \sum_{y \in \{0,1\}^k}
            \abs{v\oo{\oo{y, \cc{x_i}_{i=k+1}^{d}}}}
            \le
            2^k
            \sup_{y \in \{0,1\}^d} \abs{Q(y)}
            \le
            2^d
            \sup_{y \in \{0,1\}^d} \abs{Q(y)}.
        \end{align*}
        Similarly, we also have that 
        \begin{align*}
            \abs{\frac{\partial^{\abs{\alpha}}}{\partial x^\alpha} v(x)} \le 2^d
            \sup_{y \in \{0,1\}^d} \abs{Q(y)}
        \end{align*}
        for all $\alpha \in \oo{\mathbb{N}_0}^d$.
    \end{proof}

%    Next, we will show that the conversion from discrete problem to the continuous problem is valid, given the connection of the solution of one to another.

%    \begin{proposition}
%    \label{proposition:soltuions}
%        If $y^* \sin \armax$
%    \end{proposition}
%    \begin{proof}
        
%    \end{proof}

\section{1-Dimensional Distribution}
\label{appendix:1ddistribution}

    \begin{definition}
        A distribution $\sigma \in \Delta\oo{\mathbb{R}}$ is monotone in the interior region if, for any $z, z' \in \mathbb{R}$ with $z' >z$ and $\sigma(z) >0$, we have that $\sigma(z') > \sigma(z)$.
    \end{definition}
    \begin{corollary}
        If a distribution $\sigma$ is monotone in the interior region, then $\text{support}(\sigma)$ is a closed interval.
    \end{corollary}
    \begin{proof}
        Let $\ubar{z}$ be $\inf\oo{\left\{z\in \mathbb{R}: \sigma(z) > 0\right\}}$, and $\bar{z}$ be $\sup\oo{\left\{z\in \mathbb{R}: \sigma(z) < 1\right\}}$. Thus, $\text{support}(\sigma) \subseteq \cc{\ubar{z}, \bar{z}}$. Next, we will show that $\text{support}(\sigma) = \cc{\ubar{z}, \bar{z}} \cap \mathbb{R}$. 

        Since the support is closed by definition, we will have that, if $\text{support}(\sigma) \ne \cc{\ubar{z}, \bar{z}} \cap \mathbb{R}$, then, there exists a non-zero measure closed interval $[a,b] \subseteq \oo{\ubar{z}, \bar{z}}$ such that $[a,b] \cap \text{support}(\sigma) = \emptyset$. However, since $\sigma(a) >0$, so $\sigma(b)> \sigma(a)$, contradicting with $[a,b] \cap \text{support}(\sigma) = \emptyset$. 

        Therefore, $\text{support}(\sigma) = \cc{\ubar{z}, \bar{z}} \cap \mathbb{R}$, which is a closed interval.
    \end{proof}

    \begin{definition}
        A distribution $\sigma \in \Delta\oo{\mathbb{R}}$ is non-atomic, if there does not exists any $z \in \mathbb{R}$ with $\sigma\oo{\{z\}} > 0$.
    \end{definition}
    \begin{corollary}
        If a distribution $\sigma$ is non-atomic, then the CDF $\sigma$ is continuous.
    \end{corollary}
    \begin{proof}
        Since CDF is right-continuous by definition, it then suffices to show that $\sigma$ is left-continuous. This is true since, for any $z \in \mathbb{R}$,
        \begin{align*}
            \lim_{z' \to z^{-}} \sigma(z)
            =
            \sigma\oo{\bigcup_{z' < z} \oc{-\infty, z'}}
            =
            \sigma\oo{\oo{-\infty, z}}
            =
            \sigma\oo{\oc{-\infty, z}}- \sigma\oo{\{ z\}}
            =
            \sigma(z)- 0.
        \end{align*}
    \end{proof}

    \begin{definition}
    \label{def:invertible}
        A distribution $\sigma \in \Delta\oo{\mathbb{R}}$ is well-invertible if, there exists a function $f: (0,1) \to \mathbb{R}$ such that 
        \begin{align*}
            z = f(\sigma(z))
        \end{align*}
        for any $z \in \text{int}\oo{\text{support}(\sigma)}$, and
        \begin{align*}
            x = \sigma\oo{f(x)}
        \end{align*}
        for any $x \in (0,1)$.
    \end{definition}
    
    \begin{proposition}
    \label{proposition:inversecondition}
        A distribution $\sigma \in \Delta\oo{\mathbb{R}}$ is well-invertible if and only if $\sigma$ is monotone in the interior region and non-atomic.
    \end{proposition}
    \begin{proof}
        We first prove the backward direction. Assume that $\sigma$ is monotone in the interior region and non-atomic. Define $\oo{\ubar{z}, \bar{z}} = \oo{\inf(\supp(\sigma)), \sup(\supp(\sigma))}$, and a function $\hat{\sigma}: \oo{\ubar{z}, \bar{z}} \to [0,1]$ such that $\hat{\sigma}$ agrees with $\sigma$ everywhere in its domain. We will have that $\hat{\sigma}$ is an increasing function with range being $(0,1)$. Thus, by redefining $\hat{\sigma}: \oo{\ubar{z}, \bar{z}} \to (0,1)$, we will have that $\hat{\sigma}$ is bijective, and a function $\hat{\sigma}^{-1}$ is well-defined in a conventional inverse function sense. By choosing $f$ to be $\hat{\sigma}^{-1}$, we will have that $\sigma$ is invertible.

        Next, we prove the forward direction. We assume that $\sigma$ is well-invertible. Since the range of $\sigma$ is a superset of $(0,1)$ if and only if $\sigma$ is non-atomic, so it is non-atomic. Assume that $\sigma$ is not monotone in the interior. There exists $(a,b) \in \mathbb{R}^2$ with $b>a$ such that $\sigma(a) = \sigma(b) \in (0,1)$. Since $\sigma$ is non-atomic, we have that $a,b \in \interior(\supp(\sigma))$, but
        \begin{align*}
            a = f(\sigma(a)) = f(\sigma(b)) = b > a,
        \end{align*}
        creating a contradiction. Thus, $\sigma$ is monotone in the interior.
    \end{proof}
    Note that, if we only require the first condition, then the definition will be different. For example, there can be a point mass at $\inf(\supp(\sigma))$ and $\sup(\supp(\sigma))$. Moreover, we will also have that, we can uniquely define an ``inverse" CDF function $\sigma^{-1}$ or be the unique choice of $f$ that satisfies the condition in the definition~\ref{def:invertible}.
    \begin{proposition}
    \label{proposition:inverseCDF}
        If a distribution $\sigma$ is well-invertible, then there exists a unique $f: (0,1) \to \mathbb{R}$ satisfying the condition in the definition~\ref{def:invertible}. We denote such function $f$ as $\sigma^{-1}$ for notational simplicity, even when it is not necessarily the inverse of the function $\sigma$ in a conventional sense. Moreover, $\sigma^{-1}$ is an increasing function.
    \end{proposition}
    \begin{proof}
        From the proof of the proposition~\ref{proposition:inversecondition}, we have that a necessary condition is that
        \begin{align*}
            z = f\oo{\hat{\sigma}(z)}
        \end{align*}
        for any $z \in \interior(\supp(\sigma))$, and 
        \begin{align*}
            x = \hat{\sigma}(f(x))
        \end{align*}
        for any $x \in (0,1)$. Since $\hat{\sigma}$ is a bijection from $\interior(\supp(\sigma))$ to $(0,1)$, we then have that the only function $f$ satisfying this condition is $\hat{\sigma}^{-1}$. We also have shown in the proof of the proposition~\ref{proposition:inversecondition}, that choosing $f$ being $\hat{\sigma}^{-1}$ will satisfy the conditions in the definition, so we have that $\hat{\sigma}^{-1}$ is a unique function $f$. We then denote it as $\sigma^{-1}$ for notational simplicity.

        $\sigma^{-1}$ is increasing function, since $\hat{\sigma}$ is increasing.
    \end{proof}

    \begin{lemma}
    \label{lemma:inverseinverseCDF}
        If a distribution $\sigma$ is well-invertible, for the corresponding $\sigma^{-1}$ in the proposition~\ref{proposition:inverseCDF}, we have that there exists a unique distribution $\tilde{\sigma}$ such that 
        \begin{align*}
            x = \tilde{\sigma}(\sigma^{-1}(x))
        \end{align*}
        for any $x \in (0,1)$.
        Moreover, such $\tilde{\sigma} = \sigma$.
    \end{lemma}
    \begin{proof}
        Note that $\tilde{\sigma} = \sigma$ works from the proposition~\ref{proposition:inverseCDF}.

        By choosing $x = \sigma(z)$, we can get any $x \in (0,1)$ by choosing $z \in \interior\oo{\supp(\sigma)}$, and vice versa. Thus, we have that the condition becomes
        \begin{align*}
            \sigma(z) = \tilde{\sigma}(z)
        \end{align*}
        for all $z \in \interior\oo{\supp(\sigma)}$.
        Note that $\lim_{z \to \oo{\inf\oo{\supp(\sigma)}}^+} \sigma(z) = 0$, and $\lim_{z \to \oo{\sup\oo{\supp(\sigma)}}^-} \sigma(z) = 1$, so we have that $\tilde{\sigma}(z) = 0$ for all $z \le \inf\oo{\supp(\sigma)}$, and $\tilde{\sigma}(z) = 1$ for all $z \ge \sup\oo{\supp(\sigma)}$.

        Thus, $\sigma(z) = \tilde{\sigma}(z)$ for all $z \in \mathbb{R}$.
    \end{proof}

    From this, we can then prove the following lemma.
    \begin{lemma}
    \label{lemma:calibrated}
        For any well-invertible distribution $\sigma$ as defined in the definition~\ref{def:invertible}, $\epsilon \sim \sigma^d$ if and only if 
        \begin{align*}
            \cc{\mathbf{1}_{\sigma^{-1}(x_i)-\epsilon_i \ge 0}}_{i=1}^d \sim P_x
        \end{align*}
        for any $x \in (0,1)^d$, where $\sigma^{-1}$ is the inverse CDF as defined in the proposition~\ref{proposition:inverseCDF}.
    \end{lemma}
        \begin{proof}
        We first proof the forward direction.
        Consider any arbitrary $i \in \{1,2,\dots, d\}$.
        From the definition~\ref{def:invertible} and proposition~\ref{proposition:inversecondition}, we have that
        \begin{align*}
            \prob{\sigma^{-1}(x_i) - \epsilon_i \ge 0}
            =
            \prob{\epsilon_i \le \sigma^{-1}(x_i)}
            =
            \sigma\oo{\sigma^{-1}(x_i)} = x_i.
        \end{align*}
        From independence of $\left\{\epsilon_i\right\}_{i=1}^d$, we then have that $\cc{\mathbf{1}_{\sigma^{-1}(x_i)-\epsilon_i \ge 0}}_{i=1}^d \sim P_x$.
        
        For the backward direction, we first show that, in 1-dimensional case, by denoting the distribution of $\epsilon$ to be $\tilde{\sigma}$,
        \begin{align*}
            \mathbf{1}_{\sigma^{-1}(x) - \epsilon \ge 0} \sim \text{Bernoulli}(x)
        \end{align*}
        for all $x \in (0,1)$ if and only if 
        \begin{align*}
            \tilde{\sigma}\oo{\sigma^{-1}(x)} = x
        \end{align*}
        for all $x \in (0,1)$. From the lemma~\ref{lemma:inverseinverseCDF}, we then have that $\tilde{\sigma} = \sigma$. 

        In the multi-dimensional case, the marginal of each coordinate suggests that $\epsilon_i$ distributed according to $\sigma$. If $\{\epsilon_i\}_{i=1}^d$ is not independent, then $\left\{\mathbf{1}_{\sigma^{-1}(x_i)-\epsilon_i \ge 0}\right\}_{i=1}^d$ is not independent, creating a contradiction.
    \end{proof}

    \begin{definition}
    \label{def:pospdf}
        A distribution $\sigma$ is bounded positively-differentiable in the interior region if there exists some $H>0$ such that
        \begin{align*}
            \sigma'(z) \in \oc{0,H}
        \end{align*}
        for all $z \in \interior(\supp(\sigma))$.
    \end{definition}
    Note that this is stronger than well-inversibility and differentiability. Moreover, we only require it to be bounded away from $\infty$ but not from $0$. This is because we are interested in the derivative of $\sigma^{-1}$, and want it to be bounded away from $0$.

    \begin{corollary}
        If a distribution $\sigma$ is positively-differentiable in the interior region, then it is well-invertible, and $\sigma^{-1}: (0,1) \to \mathbb{R}$ as defined in the proposition~\ref{proposition:inverseCDF} is differentiable everywhere and
        \begin{align*}
            \oo{\sigma^{-1}}'(x) = \frac{1}{\sigma'\oo{\sigma^{-1}(x)}}.
        \end{align*}
    \end{corollary}
    \begin{proof}
        From the proof of the proposition~\ref{proposition:inversecondition}, we have that $\oo{\sigma^{-1}}'(x) = \frac{1}{\sigma'\oo{\sigma^{-1}(x)}}$ when restrict to the region where $\sigma'\oo{\sigma^{-1}(x)} \ne 0$, which is every $x \in (0,1)$ by the definition~\ref{def:pospdf}.
    \end{proof}

    \begin{proof}
        We first proof the forward direction.
        Consider any arbitrary $i \in \{1,2,\dots, d\}$.
        From the definition~\ref{def:invertible} and proposition~\ref{proposition:inversecondition}, we have that
        \begin{align*}
            \prob{\sigma^{-1}(x_i) - \epsilon_i \ge 0}
            =
            \prob{\epsilon_i \le \sigma^{-1}(x_i)}
            =
            \sigma\oo{\sigma^{-1}(x_i)} = x_i.
        \end{align*}
        From independence of $\left\{\epsilon_i\right\}_{i=1}^d$, we then have that $\cc{\mathbf{1}_{\sigma^{-1}(x_i)-\epsilon_i \ge 0}}_{i=1}^d \sim P_x$.
        
        For the backward direction, we first show that, in 1-dimensional case, by denoting the distribution of $\epsilon$ to be $\tilde{\sigma}$,
        \begin{align*}
            \mathbf{1}_{\sigma^{-1}(x) - \epsilon \ge 0} \sim \text{Bernoulli}(x)
        \end{align*}
        for all $x \in (0,1)$ if and only if 
        \begin{align*}
            \tilde{\sigma}\oo{\sigma^{-1}(x)} = x
        \end{align*}
        for all $x \in (0,1)$. From the lemma~\ref{lemma:inverseinverseCDF}, we then have that $\tilde{\sigma} = \sigma$. 

        In the multi-dimensional case, the marginal of each coordinate suggests that $\epsilon_i$ distributed according to $\sigma$. If $\{\epsilon_i\}_{i=1}^d$ is not independent, then $\left\{\mathbf{1}_{\sigma^{-1}(x_i)-\epsilon_i \ge 0}\right\}_{i=1}^d$ is not independent, creating a contradiction.
    \end{proof}

\section{Deferred Proofs}

\subsection{Deferred Proofs from the Section~\ref{section:sgrad}}

    \paragraph{Proof for the Proposition~\ref{proposition:seval_characterization}}
    \begin{proof}
        We first prove the forward direction.
        For any $Q: \{0,1\}^d \to \mathbb{R}$, any $x \in (0,1)^d$,
        \begin{align*}
            \sum_{y \in \{0,1\}^d}\oo{Q(y) \prod_{i=1}^d\cc{x_iy_i+(1-x_i)(1-y_i)}}
            &= v(x; Q)\\
            &= \expect{V(x; \omega, Q)}\\
            &= \expect{f_{\text{v}}(x, \omega, Q(K(x; \omega)))}\\
            &= \expect{\sum_{y \in \{0,1\}^d}\oo{f_{\text{v}}(x, \omega, Q(y))\mathbf{1}_{K(x; \omega) = y}}}\\
            &= \sum_{y \in \{0,1\}^d} \expect{\oo{f_{\text{v}}(x, \omega, Q(y))\mathbf{1}_{K(x; \omega) = y}}}.
        \end{align*}
        For any $y \in \{0,1\}^d$, we can set $Q(y') = 0$ for all $y' \in \{0,1\}^d -\{y\}$, and get that
        \begin{align*}
            Q(y) \prod_{i=1}^d\cc{x_iy_i+(1-x_i)(1-y_i)}
            =
            \expect{\oo{f_{\text{v}}(x, \omega, Q(y))\mathbf{1}_{K(x; \omega) = y}}}.
        \end{align*}

        For the backward direction, we easily have that
        \begin{align*}
            \expect{V(x; \omega, Q)}
            &=
            \sum_{y \in \{0,1\}^d} \expect{\oo{f_{\text{v}}(x, \omega, Q(y))\mathbf{1}_{K(x; \omega) = y}}}\\
            &=
            \sum_{y \in \{0,1\}^d}\oo{Q(y) \prod_{i=1}^d\cc{x_iy_i+(1-x_i)(1-y_i)}}\\
            &= v(x; Q).
        \end{align*}
    \end{proof}

\subsection{Proof for the Theorem~\ref{theorem:easy}}

    Following the proposition~\ref{proposition:seval_characterization}, we can get an easy sufficient condition for a single-queried stochastic evaluation function $V$ to be gradientable.
    \begin{corollary}
    \label{corollary:encoder}
        For any stochastic evaluation function $V$ with respect to $\oo{\Omega, \mathcal{F}, \mathbb{P}}$ with $f_v$ and $K$ as in the definition~\ref{def:seval}, there exist some $d' \in \mathbb{N}$, $\mathcal{Z} \subseteq \mathbb{R}^{d'}$ being a set product of $d'$ open-interval, and bijective differentiable encoding function $e:(0,1)^d \to \mathcal{Z}$ such that there exist functions $\tilde{f}_{\text{v}}: \mathcal{Z} \times \Omega \times \mathbb{R} \to \mathbb{R}$ and $\tilde{K}: \mathcal{Z} \times \Omega \to \{0,1\}^d$ such that, for all $x \in (0,1)^d$, $q \in \mathbb{R}$, $y \in \{0,1\}^d$, $\omega \in \Omega$,
        \begin{align*}
            \tilde{f}_{\text{v}}(e(x), \omega, q) = f_{\text{v}}(x, \omega, q),
        \end{align*}
        and
        \begin{align*}
            \tilde{K}(e(x); \omega) = K(x; \omega).
        \end{align*}
    \end{corollary}
    \begin{proof}
        This is obvious since we can choose $e(x) = x$ for all $x \in (0,1)^d$, making the Jacobian $D_x e(x) = I$.
    \end{proof}
    
    \begin{lemma}
    \label{lemma:whengradient}
        Consider a single-queried stochastic evaluation function $V$ with respect to $\oo{\Omega, \mathcal{F}, \mathbb{P}}$ with some $\oo{e, \mathcal{Z},\tilde{f}_{\text{v}}, \tilde{K}}$ as in the corollary~\ref{corollary:encoder}. 
        
        If, for every arbitrary $q \in \mathbb{R}$ and $y \in \{0,1\}^d$,
        \begin{enumerate}
            \item For $\mathbb{P}$-almost surely $\omega \in \Omega$, the mapping
            \begin{align*}
                z \in \mathcal{Z} \mapsto \tilde{f}_{\text{v}}(z, \omega, q)\mathbf{1}_{\tilde{K}(z;\omega)= y}
            \end{align*}
            is absolutely continuous;
            \item For every $z \in \mathcal{Z}$, for $\mathbb{P}$-almost surely $\omega \in \Omega$,
            \begin{align*}
                \nabla_z\oo{\tilde{f}_{\text{v}}(z, \omega, q)\mathbf{1}_{\tilde{K}(z;\omega)= y}}
            \end{align*}
            exists;
            \item There exists a function $g: \Omega \to \mathbb{R} \cup \{\infty\}$ that is integrable with respect to $\oo{\Omega, \mathcal{F}, \mathbb{P}}$, and
            \begin{align*}
                \norm{\nabla_{z'}\oo{\tilde{f}_{\text{v}}(z', \omega, q)\mathbf{1}_{\tilde{K}(z';\omega)= y}}}_{\infty} \le g(\omega)
            \end{align*}
            for all $\omega \in \Omega$ and all $z \in \mathcal{Z}$ where the gradient is well-defined,
        \end{enumerate}
        then $V$ is a single-queried gradientable stochastic evaluation function, and
        \begin{align*}
            G(x; \omega, Q) :=
            \cc{\sum_{y\in \{0,1\}^d} \oo{\mathbf{1}_{\tilde{K}(e(x);\omega)= y}
            \nabla_{z = e(x)}\tilde{f}_{\text{v}}(z, \omega, Q(y))
            }}^\top D_x e(x)
        \end{align*}
        is a single-queried stochastic gradient function, both with respect to $\oo{\Omega, \mathcal{F}, \mathbb{P}}$.
    \end{lemma}
    \begin{proof}
        We first have that 
        \begin{align*}
            \expect{\nabla_z\oo{\tilde{f}_{\text{v}}(z, \omega, q)\mathbf{1}_{\tilde{K}(z;\omega)= y}}} = 
            \nabla_z \expect{\tilde{f}_{\text{v}}(z, \omega, q)\mathbf{1}_{\tilde{K}(z;\omega)= y}}
        \end{align*}
        for all $z \in \mathcal{Z}$ by the lemma~\ref{lemma:gradientswap}, then we can apply the corollary~\ref{corollary:sgradfromdiffsval} to get the result.
    \end{proof}
    Although this lemma may seem demanding, we can consider the main condition to be that the mapping is continuous and gradient is well-defined for all $x\in (0,1)^d$ for almost surely $\omega \in \Omega$, the encoding (in the corollary~\ref{corollary:encoder}) and other conditions in the lemma~\ref{lemma:whengradient} are mostly for regularity condition. For example, the class of stochastic evaluation $V(x; \omega) = Q(K(x; \omega))$ will not satisfy the continuity condition. Note that this is also a sufficient condition and is not a necessary condition. We can weaken the global gradient domination to difference quotient domination.

    From the lemma~\ref{lemma:whengradient}, the proof for the theorem~\ref{theorem:easy} is straightforward and written down below.
    \begin{proof}
        For any $x \in (0,1)^d$, any $q \in \mathbb{R}$, and any $y \in \{0,1\}^d$, we have that
        \begin{align*}
            \expect{f_{\text{v}}(x, \omega, q) \mathbf{1}_{K(x;\omega) = y}} 
            &=
            q
            \mexpect{\epsilon_i \overset{\text{iid}}{\sim} \sigma}{\oo{\prod_{i=1}^d f\oo{\abs{\hat{\sigma}^{-1}(x_i) +\epsilon_i}}} \mathbf{1}_{\cc{\mathbf{1}_{\hat{\sigma}^{-1}(x_i) +\epsilon_i \ge 0}}_{i=1}^d = y}}\\
            &=
            q
            \mexpect{\epsilon_i \overset{\text{iid}}{\sim} \sigma}{\oo{\prod_{i=1}^d f\oo{\abs{\hat{\sigma}^{-1}(x_i) +\epsilon_i}}} 
            \prod_{i=1}^d \mathbf{1}_{\mathbf{1}_{\hat{\sigma}^{-1}(x_i) +\epsilon_i \ge 0} = y_i}}\\
            &=
            q
            \prod_{i=1}^d
            \mexpect{\epsilon \sim \sigma}{ f\oo{\abs{\hat{\sigma}^{-1}(x_i) +\epsilon_i}}\mathbf{1}_{\mathbf{1}_{\hat{\sigma}^{-1}(x_i) +\epsilon_i \ge 0} = y_i}}.
        \end{align*}

        From that $f(z) = 0$ for all $z\le 0$, we have that, if $y_i = 1$, then
        \begin{align*}
            \mexpect{\epsilon \sim \sigma}{ f\oo{\abs{\hat{\sigma}^{-1}(x_i) +\epsilon}}\mathbf{1}_{\mathbf{1}_{\hat{\sigma}^{-1}(x_i) +\epsilon \ge 0} = y_i}}
            &=
            \mexpect{\epsilon \sim \sigma}{ f\oo{\abs{\hat{\sigma}^{-1}(x_i) +\epsilon}}\mathbf{1}_{\hat{\sigma}^{-1}(x_i) +\epsilon \ge 0}}\\
            &=
            \mexpect{\epsilon \sim \sigma}{ f\oo{\hat{\sigma}^{-1}(x_i) +\epsilon}}\\
            &=
            x_i\\
            &= x_iy_i + (1-x_i)(1-y_i),
        \end{align*}
        and, if $y_i = 0$, then
        \begin{align*}
            \mexpect{\epsilon \sim \sigma}{ f\oo{\abs{\hat{\sigma}^{-1}(x_i) +\epsilon}}\mathbf{1}_{\mathbf{1}_{\hat{\sigma}^{-1}(x_i) +\epsilon \ge 0} = y_i}}
            &=
            \mexpect{\epsilon \sim \sigma}{ f\oo{\abs{\hat{\sigma}^{-1}(x_i) +\epsilon}}\mathbf{1}_{\hat{\sigma}^{-1}(x_i) +\epsilon < 0}}\\
            &=
            \mexpect{\epsilon \sim \sigma}{ f\oo{\abs{\hat{\sigma}^{-1}(x_i) +\epsilon}}\mathbf{1}_{\hat{\sigma}^{-1}(x_i) +\epsilon \le 0}} + \mexpect{\epsilon \sim \sigma}{ f\oo{\abs{0}}\mathbf{1}_{\hat{\sigma}^{-1}(x_i) +\epsilon = 0}}\\
            &=
            \mexpect{\epsilon \sim \sigma}{ f\oo{\abs{\hat{\sigma}^{-1}(x_i) +\epsilon}}\mathbf{1}_{\hat{\sigma}^{-1}(x_i) +\epsilon \le 0}}\\
            &=
            \mexpect{\epsilon \sim \sigma}{ f\oo{\abs{-\hat{\sigma}^{-1}(x_i) +\epsilon}}\mathbf{1}_{-\hat{\sigma}^{-1}(x_i) +\epsilon \ge 0}}\\
            &=
            \mexpect{\epsilon \sim \sigma}{ f\oo{\abs{\hat{\sigma}^{-1}(1-x_i) +\epsilon}}\mathbf{1}_{\hat{\sigma}^{-1}(1-x_i) +\epsilon \ge 0}}\\
            &=
            1-x_i\\
            &= x_iy_i + (1-x_i)(1-y_i).
        \end{align*}
        Thus, we have that
        \begin{align*}
            \expect{f_{\text{v}}(x, \omega, q) \mathbf{1}_{K(x;\omega) = y}} = q\prod_{i=1}^d \cc{x_iy_i + (1-x_i)(1-y_i)}.
        \end{align*}
        From the proposition~\ref{proposition:seval_characterization}, we then have that $V$ is a single-queried stochastic evaluation function.

        Next, we will show that $V$ is gradientable y showing that the conditions in the lemma~\ref{lemma:whengradient} are satisfied by choosing the function $e$ to be $\hat{\sigma}^{-1}$. We have that, for any $z \in \interior\oo{\supp\oo{\hat{\sigma}}}$, almost surely $\omega \in \Omega$, any $q \in \mathbb{R}$, any $y \in \{0,1\}^d$,
        \begin{align*}
            \tilde{f}_{\text{v}}(z,\omega, q)\mathbf{1}_{\tilde{K}(z;\omega) = y}
            &=
            q
            \prod_{i=1}^d
            \cc{
            f\oo{\abs{z_i +\epsilon_i(\omega)}}\mathbf{1}_{\mathbf{1}_{z_i +\epsilon_i(\omega) \ge 0} = y_i}}\\
            &=
            q
            \prod_{i=1}^d
            \cc{
            y_i f\oo{z_i +\epsilon_i(\omega)} + (1-y_i) f\oo{-z_i -\epsilon_i(\omega)}}
        \end{align*}
        The function $f$ is absolutely continuous, so the mapping $z \mapsto \tilde{f}_{\text{v}}(z,\omega, q)$ is absolutely continuous for almost surely $\omega$.

        Moreover, for any $i \in \{1,2,\dots, d\}$,
        \begin{align*}
            \frac{\partial}{\partial z_i} \oo{\tilde{f}_{\text{v}}(z,\omega, q)\mathbf{1}_{\tilde{K}(z;\omega) = y}}
            &=
            \mathbf{1}_{\tilde{K}(z;\omega) = y}
            \frac{\partial}{\partial z_i} \tilde{f}_{\text{v}}(z,\omega, q)\\
            &=
            q
            \cc{
            \mathbf{1}_{\tilde{K}(z;\omega) = y}
            \prod_{j \ne i} f(\abs{z_j+ \epsilon_j(\omega)}) 
            }
            \frac{\partial}{\partial z_i} f(\abs{z_i+ \epsilon_i(\omega)}),
        \end{align*}
        which is bounded for almost surely $\omega$ by the definition~\ref{def:goodtuple}.

        The existence of gradient for all $z \in \interior\oo{\supp\oo{\hat{\sigma}}}$ almost surely $\omega$ also follows from the definition~\ref{def:goodtuple}.

        Thus, by applying the~\ref{lemma:whengradient}, we have that $V$ is gradientable, and the gradient can be computed to be in the form shown in the theorem.
    \end{proof}

\subsection{Technical Lemma}

    \begin{lemma}
    \label{lemma:gradientswap}
        Assume $v: \mathcal{X} \to \mathbb{R}$ for some $\mathcal{X}$ is a set product of $d \in \mathbb{N}$ non-empty open intervals, and $\nabla_x:\mathcal{X} \to \mathbb{R}^d$ exists and is bounded.
        For a probability space $\oo{\Omega, \mathcal{F}, \mathbb{P}}$, if $V: \mathcal{X} \times \Omega \to \mathbb{R}$ be such that
        \begin{align*}
            \expect{V(x; \omega)} = v(x)
        \end{align*}
        for all $x \in \mathcal{X}$, and the following conditions hold:
        \begin{enumerate}
            \item For $\mathbb{P}$-almost surely $\omega \in \Omega$, the function $V(\cdot; \omega)$ is absolutely continuous;
            \item For all $x \in \mathcal{X}$, the gradient $\nabla_x V(x; \omega)$ exists for $\mathbb{P}$-almost surely $\omega \in \Omega$;
            \item There exists an $L^1(\Omega)$ function $g$ such that 
            \begin{align*}
                \norm{\nabla_x V(x; \omega)}_{\infty} \le g(\omega)
            \end{align*}
            for all $x \in \mathcal{X}$, $\omega \in \Omega$,
        \end{enumerate}
        then, for all $x \in \mathcal{X}$,
        \begin{align*}
            \nabla v(x) = \expect{\nabla_x V(x; \omega)}.
        \end{align*}
    \end{lemma}
    \begin{proof}
        For any $\omega \in \Omega$, where the function $V(\cdot, \omega)$ is absolutely continuous, we have that there exists $\nabla_x V(x, \omega)$ for almost every $x$, and, for any $x \in \mathcal{X}$, and any $e \in \mathbb{R}^d$ with $\norm{e}_2 =1$, and $h > 0$ such that $x +he \in \mathcal{X}$, we will have that
        \begin{align*}
            V(x+he, \omega) - V(x, \omega) = \int_{t=0}^1 h e^\top \nabla_x V(x+the, \omega) dt.
        \end{align*}
        Since this holds for $\mathbb{P}$-almost surely $\omega \in \Omega$, 
        \begin{align*}
            \frac{v(x+he)-v(x)}{h} =
            \expect{
                \frac{V(x+he, \omega)-V(x, \omega)}{h}
            }
            =
            \expect{
            \int_{t=0}^1 e^\top \nabla_x V(x+the, \omega) dt}.
        \end{align*}
        Since $\norm{e}_2 =1$, we have that, for all $t \in [0,1]$ where the gradient exists,
        \begin{align*}
            \abs{e^\top \nabla_x V(x+the, \omega)} \le  d \norm{\nabla_x V(x+the, \omega)}_\infty \le d g(\omega).
        \end{align*}
        Thus, we have that
        \begin{align*}
            \expect{
            \int_{t=0}^1 e^\top \nabla_x V(x+the, \omega) dt}
            &=
            \int_{\omega \in \Omega}
            \int_{t=0}^1 e^\top \nabla_x V(x+the, \omega) dt \mathbb{P}(d\omega)\\
            &\le
            \int_{\omega \in \Omega}
            \int_{t=0}^1 d g(\omega) dt \mathbb{P}(d\omega)\\
            &=
            \int_{\omega \in \Omega}
             d g(\omega) \mathbb{P}(d\omega)\\
            &< \infty.
        \end{align*}
        Note that $\nabla_x V(x, \omega)$ is jointly measurable in $\mathcal{B}\oo{\mathbb{R}^d} \times \mathcal{F}$, so $e^\top \nabla_x V(x+the, \omega)$ is jointly measurable in $\mathcal{B}([0,1]) \times \mathcal{F}$. We can then apply Fubini's theorem to get that
        \begin{align*}
            \frac{v(x+he)-v(x)}{h}
            =
            \int_{t=0}^1
            \expect{
             e^\top \nabla_x V(x+the, \omega) dt},
        \end{align*}
        making
        \begin{align*}
            \lim_{h\to 0}\frac{v(x+he)-v(x)}{h}
            =
            \lim_{h\to 0}
            \int_{t=0}^1
            \expect{
             e^\top \nabla_x V(x+the, \omega) dt}.
        \end{align*}
        We further apply dominated convergence to get that
        \begin{align*}
            \lim_{h\to 0}\int_{t=0}^1\expect{e^\top \nabla_x V(x+the, \omega)} dt
            &=
            \int_{t=0}^1
            \expect{e^\top \oo{\lim_{h\to 0} \nabla_x V(x+the, \omega)}}dt\\
            &=
            \int_{t=0}^1
            \expect{e^\top \nabla_x V(x, \omega)}dt\\
            &=
            \expect{e^\top \nabla_x V(x, \omega)}.
        \end{align*}
        Thus, we have that
        \begin{align*}
            e^\top \nabla v(x)
            =
            \lim_{h\to 0}\frac{v(x+he)-v(x)}{h}
            =
            \expect{e^\top \nabla_x V(x, \omega)}
            =
            e^\top
            \expect{\nabla_x V(x, \omega)}.
        \end{align*}
        From that $\mathcal{X}$ is open, this holds for any direction $e$, so
        \begin{align*}
            \nabla v(x) =\expect{\nabla_x V(x, \omega)}.
        \end{align*}
    \end{proof}

\subsection{Deferred Proofs from the Section~\ref{section:desiderata}}
    \paragraph{Proof for the Lemma~\ref{lemma:no_fsigsig}}
    
        \begin{proof}
            First, we prove for the compact support case.
            If it exists, then $\sigma = \hat{\sigma}$, so $\sigma'$ exists. Thus, $\oo{f \ast \sigma'} (z) = \sigma(z)$ for all $z \in \interior(\supp(\sigma))$.
            
            Define a function $\bar{f}(z) := \mathbf{1}_{z \ge 0}$, so $\oo{\bar{f} \ast \sigma'} (z) = \int_{x=-\infty}^z \sigma'(z) dz = \sigma(z)$. Thus, we have that
            \begin{align*}
                \oo{\oo{\bar{f} - f}\ast \sigma'}(z) = 0
            \end{align*}
            for all $z \in \interior(\supp(\sigma))$.
    
            Since $\sigma$ is bounded and symmetric, we have that $\supp(\sigma)= [-C,C]$ for some $C > 0$.
    
            Thus, we have that, for any $z \in (0,C)$, we will have that
            \begin{align*}
                \oo{\oo{\bar{f} - f}\ast \sigma'}(-z) = \int_{x=0}^{C-z} \oo{\bar{f} - f}(x) \sigma'(z+x) dx.
            \end{align*}
            Since $f(0) = 0$ and $f$ is continuous, so  there exists, some $\epsilon \in (0,C)$ such that $f(x) < \frac{1}{2}$ for all $x \in [0, \epsilon]$. Thus, $\oo{\bar{f} - f}(x) >\frac{1}{2}$ for all $x \in [0, \epsilon]$. Therefore, by choosing $z = C - \epsilon$, we will have that
            \begin{align*}
                \oo{\oo{\bar{f} - f}\ast \sigma'}(-z) = \int_{x=0}^{\epsilon} \oo{\bar{f} - f}(x) \sigma'(z+x) dx \ge \frac{1}{2} \int_{x=0}^{\epsilon} \sigma'(z+x) dx = \frac{1}{2} \sigma(-C+\epsilon) > 0,
            \end{align*}
            creating a contradiction.
            
            The proof for decreasing hazard rate case is almost identical.
        \end{proof}

\section{Experiment Results}
\label{appendix:result}

    \subsection{Result: Symmetric-Slice Optimization from the subsection~\ref{subsection:symmetricslice}}
    \label{appendix:symmetricslice}

    \begin{figure}
        \centering
        \begin{subfigure}
                \centering
                \includegraphics[width=0.8\textwidth]{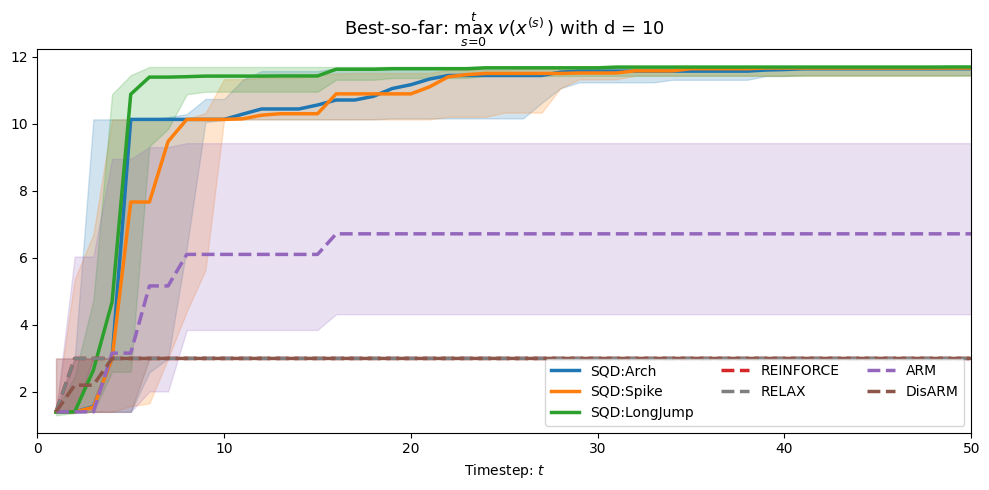}
        \end{subfigure}
        \begin{subfigure}
                \centering
                \includegraphics[width=0.8\textwidth]{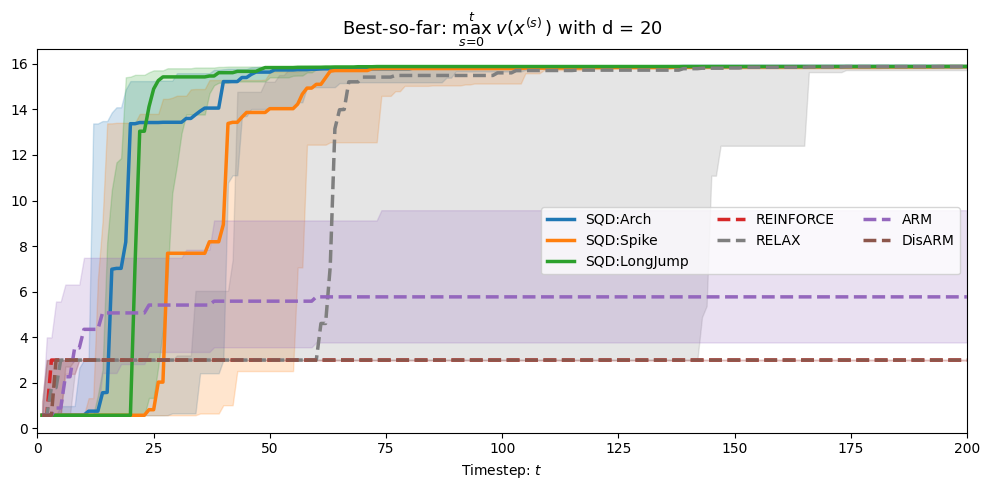}
        \end{subfigure}
        \begin{subfigure}
                \centering
                \includegraphics[width=0.8\textwidth]{figure/SymmetricSlice/30.png}
        \end{subfigure}
        \caption{Result for the experiment for symmetric-slice optimization from~\ref{subsection:symmetricslice}: Median of the best-so-far progress of different algorithms with inter-quartiles uncertainty. The algorithm we proposed are in bold line, and the pre-existing algorithms are in a dashed line. The experiment comes from $20$ trials. Note that the time-scale for the 3 plots (with 3 different dimensions) are different.}
    \end{figure}

    For the experiment comparing the SQD with Encoded SQD, we will have that the encoded version, except for the LongJump whose encoding function is linear, tends to have a more gradual progress. Both methods, SQD and Encoded SQD, are not stuck at the local minimum.
    \begin{figure}
        \centering
        \begin{subfigure}
                \centering
                \includegraphics[width=0.8\textwidth]{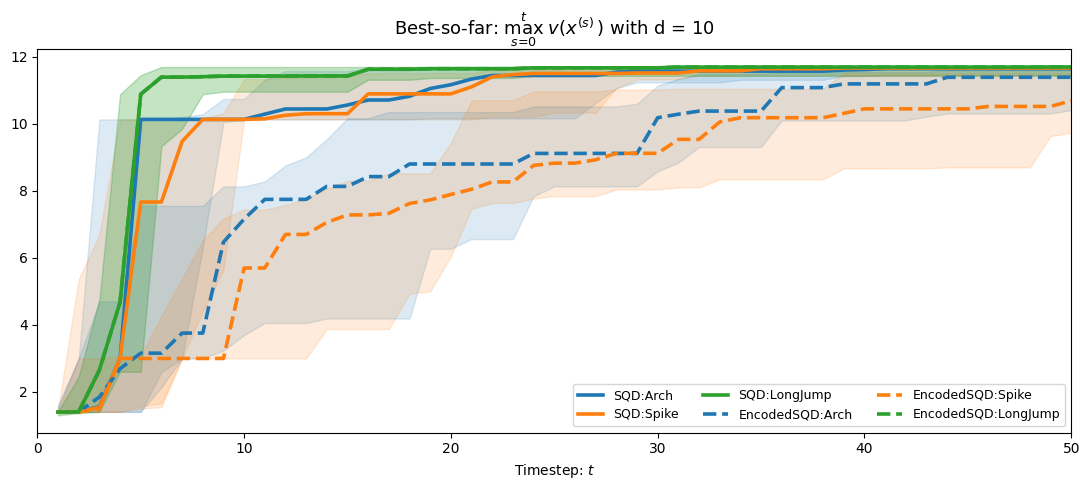}
        \end{subfigure}
        \begin{subfigure}
                \centering
                \includegraphics[width=0.8\textwidth]{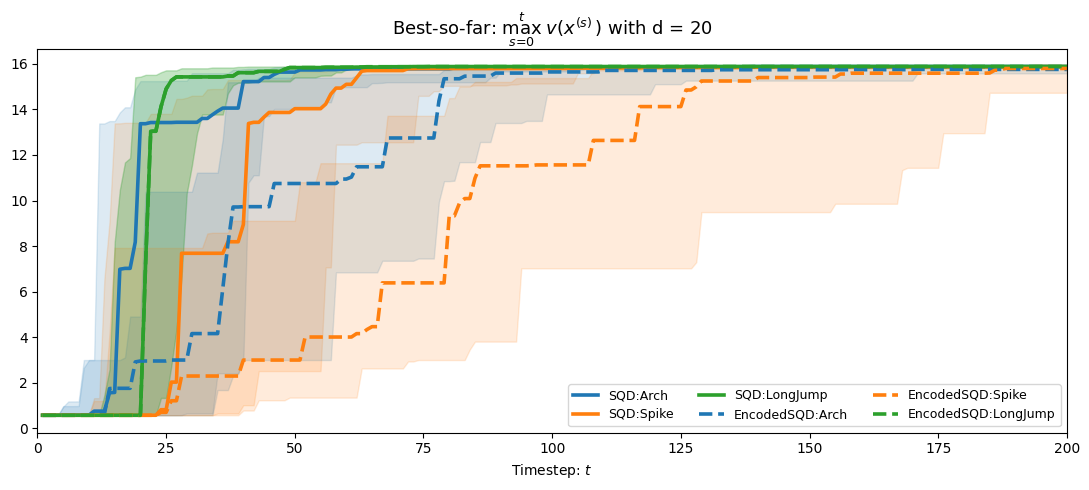}
        \end{subfigure}
        \begin{subfigure}
                \centering
                \includegraphics[width=0.8\textwidth]{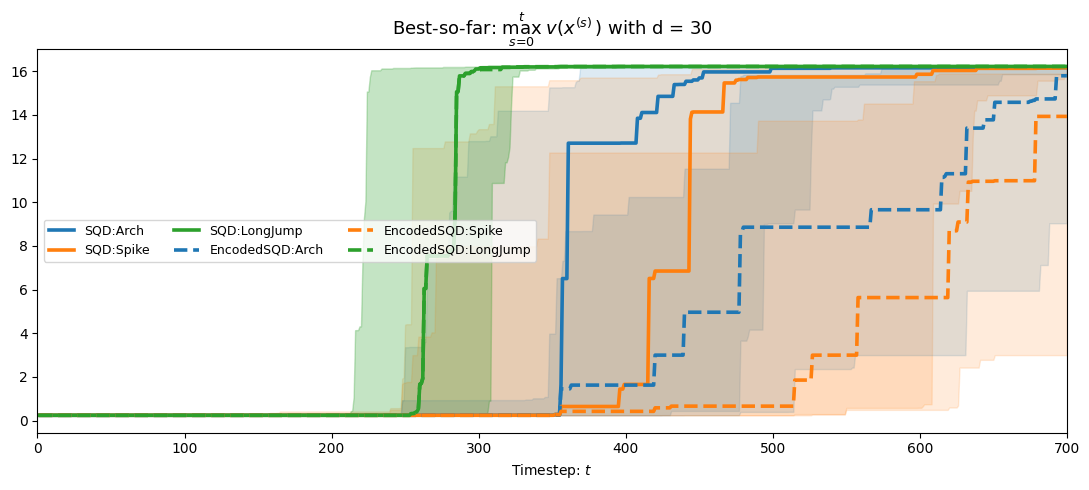}
        \end{subfigure}
        \caption{Result for the experiment for symmetric-slice optimization from~\ref{subsection:symmetricslice}: The bold lines are our SQD algorithms, while the dashed lines are our Encoded SQD algorithms. The encoding function $\hat{\sigma}^{-1}\oo{\cdot}$ is linear in the case of ``SQD: LongJump", so the two methods, ``SQD: LongJump" and `EncodedSQD: LongJump" have roughly the same performance.}
    \end{figure}

    \subsection{Modified Knapsack Experiment}
    \label{appendix:knapsack}
    Here, we provide some preliminary results for modified Knapsack problem. In this problem, we randomly assign weights $w$ from $\{1,2,\dots,9\}$ uniformly and independently to each index $i \in \{1,2,\dots, d\}$. 

    We denote the sum of elements chosen given $y \in \{0,1\}^d$ to be $S(y) = \sum_{i=1}^d y_iw_i$, and define the payoff to be
    \begin{align*}
        Q(y)=
        \begin{cases}
        20 & \text{if } T-2 \le S(y)\le T+2\\
        -5 & \text{ if } S(y) > T+2\\
        0 & \text{ if } S(y) < T-2
        \end{cases}
    \end{align*}
    where the target $T = \floor{\frac{1}{2}\sum_{i=1}^d w_i}$.

    \begin{figure}
        \centering
        \begin{subfigure}
                \centering
                \includegraphics[width=0.8\textwidth]{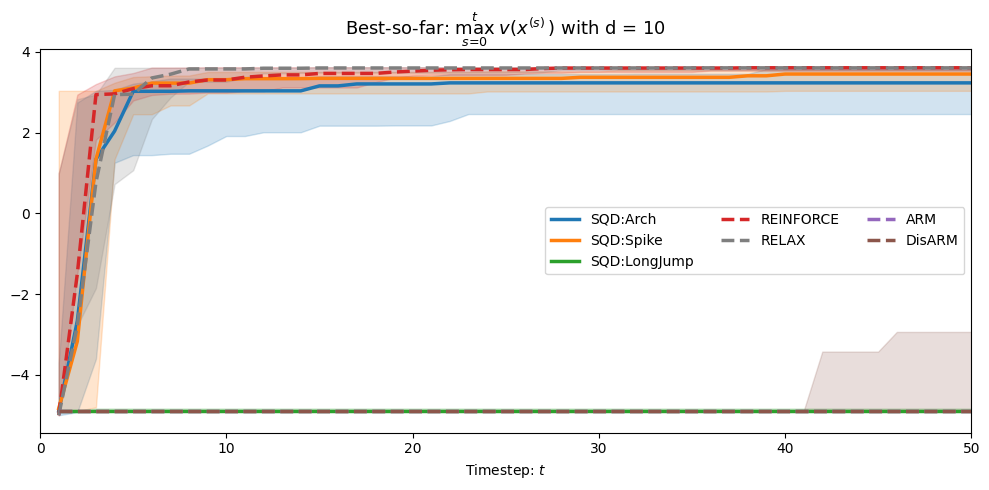}
        \end{subfigure}
        \begin{subfigure}
                \centering
                \includegraphics[width=0.8\textwidth]{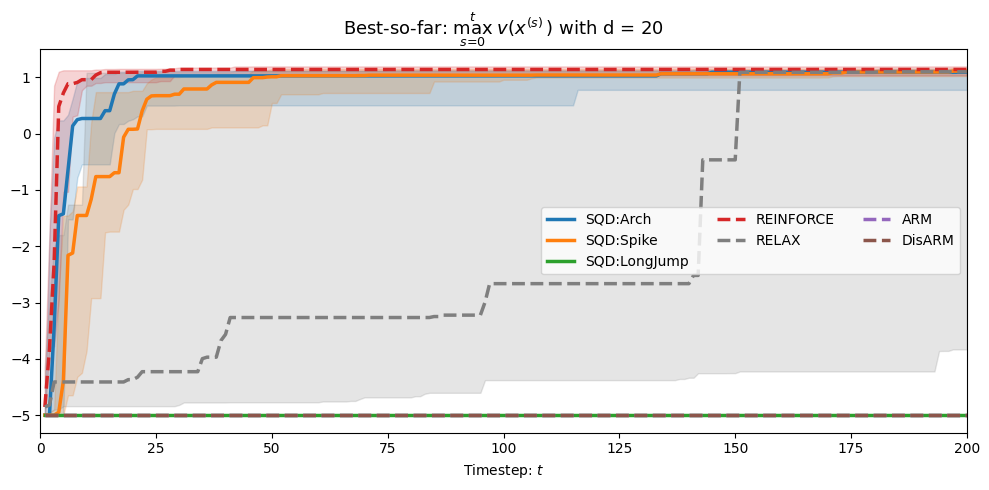}
        \end{subfigure}
        \begin{subfigure}
                \centering
                \includegraphics[width=0.8\textwidth]{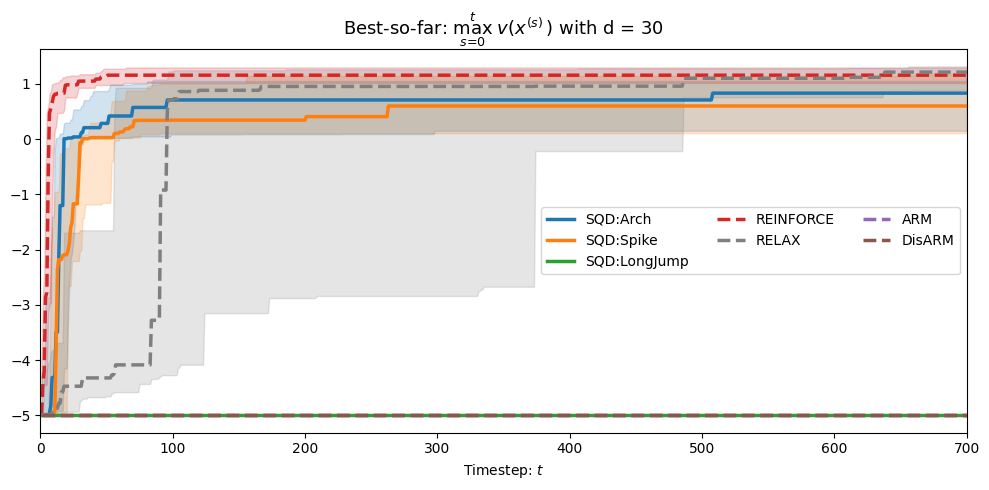}
        \end{subfigure}
        \caption{Result for the experiment for modified Knapsack problem: Median of the best-so-far progress of different algorithms with inter-quartiles uncertainty. The algorithm we proposed are in bold line, and the pre-existing algorithms are in a dashed line. The experiment comes from $20$ trials. Note that the time-scale for the 3 plots (with 3 different dimensions) are different.}
    \end{figure}

\section{Comparison between ESG \& Gradient in REINFORCE}
\label{appendix:ESGREINFORCE}
    
    In the figure~\ref{fig:comparison}, we provide a side-by-side comparison between the computational graphs of ESG algorithm~\ref{alg:ESG} and REINFORCE gradient. As point out in the session~\ref{section:desiderata}, our algorithm ESG requires that $\sigma \ne \hat{\sigma}$, leading to miscalibration. REINFORCE is calibrated, so $\sigma$ is used for both encoding (through $\sigma^{-1}(\cdot)$) and noise sampling. 

    Another main difference is that, REINFORCE gradient is from taking a gradient of a proxy $\hat{v}$ in the figure instead of the unbiased estimator $v$. However, such $\hat{v}$ we take a gradient upon does not have much connection to the actual value function.

    \begin{figure}
        \centering
        \begin{subfigure}
            \centering
            \includegraphics[width=0.9\textwidth]{figure/ESG.png}
            \caption*{(a) Computational Graph for ESG Algorithm~\ref{alg:ESG}}
        \end{subfigure}
        \begin{subfigure}
            \centering
            \includegraphics[width=0.9\textwidth]{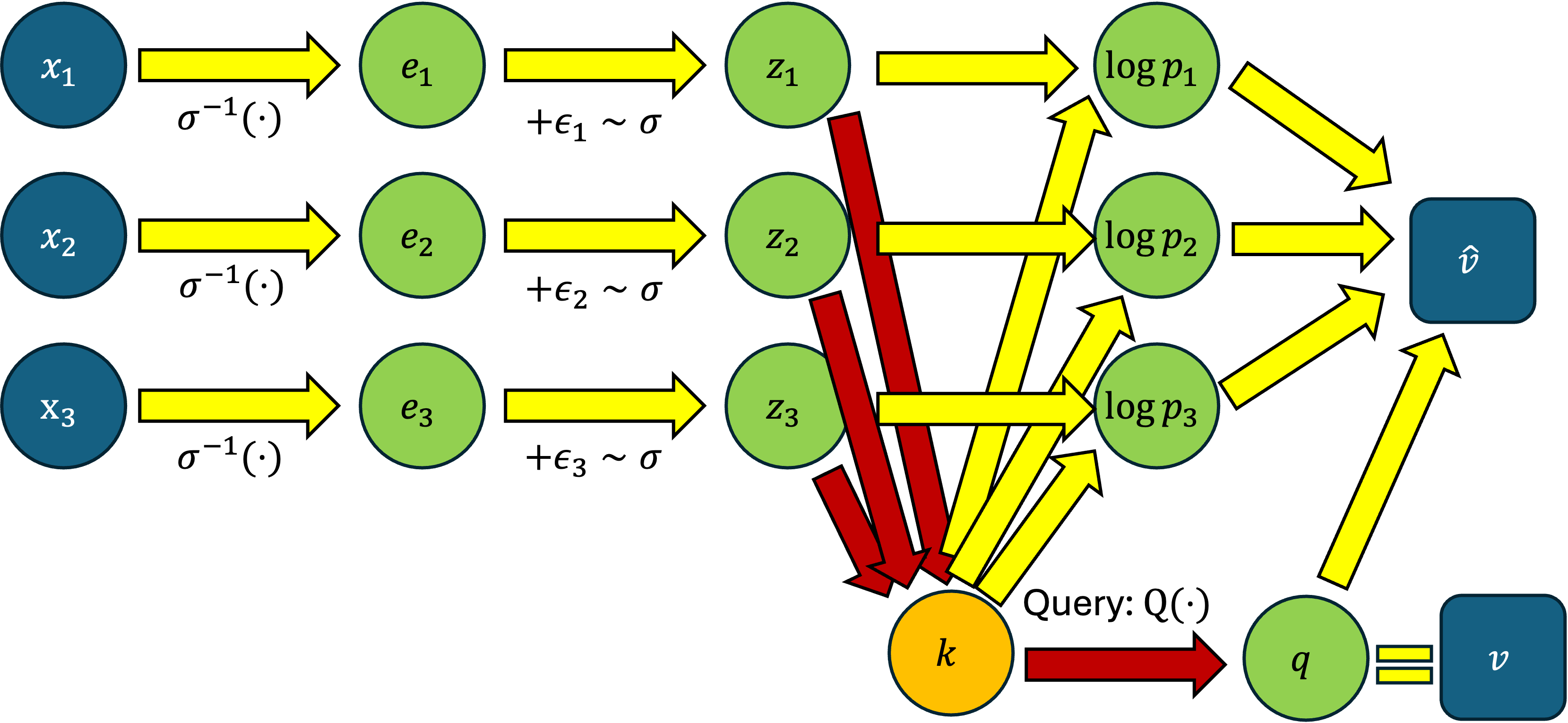}
            \caption*{(b) Computational Graph for REINFORCE Gradient}
        \end{subfigure}
        \caption{Comparison between Computational Graphs of ESG and REINFORCE Gradient}
        \label{fig:comparison}
    \end{figure}

\end{document}